\documentclass[manuscript,screen]{acmart}

\usepackage[ruled,vlined,linesnumbered]{algorithm2e}
\usepackage{graphicx}
\usepackage{capt-of}    
\usepackage{tabularx}   %
\usepackage{multicol}
\usepackage{multirow}

\usepackage[export]{adjustbox}
\usepackage{subfig}

\AtBeginDocument{%
  \providecommand\BibTeX{{%
    \normalfont B\kern-0.5em{\scshape i\kern-0.25em b}\kern-0.8em\TeX}}}

\setcopyright{acmcopyright}
\copyrightyear{2024}
\acmYear{2024}
\acmDOI{XXXXXXX.XXXXXXX}

\acmConference[ACM Transaction on Recommender Sysems]{}{}{}
\acmPrice{15.00}
\acmISBN{978-1-4503-XXXX-X/18/06}

\definecolor{blue(pigment)}{rgb}{0,0 , 0}

\newcommand{\updated}[1]{\textcolor{black}{{#1}}}
\newcommand{\newtxt}[1]{\textcolor{black}{{#1}}}
\newcommand{\updatednew}[1]{\textcolor{blue(pigment)}{{#1}}}

\newcommand{\datasetA}{MovieLens\xspace}
\newcommand{\datasetB}{Epinion\xspace}
\newcommand{\datasetC}{Gowalla\xspace}
\newcommand{\datasetD}{LastFM\xspace}

\newcommand{\usergroupA}{active\xspace}
\newcommand{\usergroupB}{inactive\xspace}

\newcommand{\itemgroupA}{short-head\xspace}
\newcommand{\itemgroupB}{long-tail\xspace}

\newcommand{\dquotes}[1]{``#1''}
\newcommand{\squotes}[1]{`#1'}

\newcommand{\eg}{e.g., }
\newcommand{\ie}{i.e., }
\newcommand{\wrt}{w.r.t.~}

\newcommand{\mycc}{\color{black}}%

\newcommand{\partitle}[1]{\vspace{2mm}\noindent\textbf{#1}}

\DeclareMathOperator{\E}{\mathbb{E}}

\begin{document}

\title[Personalized Consumer and Producer Group Fairness]{A \newtxt{Personalized} Framework for Consumer and Producer Group Fairness Optimization in Recommender Systems} 


\author{Hossein A.~Rahmani}
\affiliation{%
  \institution{University College London}
  \city{London}
  \country{UK}}
\email{hossein.rahmani.22@ucl.ac.uk}

\author{Mohammadmehdi Naghiaei}
\affiliation{%
  \institution{University of Southern California}
  \city{California}
  \country{USA}}
\email{naghiaei@usc.edu}

\author{Yashar Deldjoo}
\affiliation{%
  \institution{Polytechnic University of Bari}
  \city{Bari}
  \country{Italy}}
\email{yashar.deldjoo@poliba.it}

\renewcommand{\shortauthors}{H.~A.~Rahmani, M.~Naghiaei, Y.~Deldjoo}

\begin{abstract}
In recent years, there has been an increasing recognition that when machine learning (ML) algorithms are used to automate decisions, they may mistreat individuals or groups, with legal, ethical, or economic implications. Recommender systems are prominent examples of these machine learning (ML) systems that aid users in making decisions. The majority of past literature research on RS fairness treats user and item fairness concerns independently, ignoring the fact that recommender systems function in a two-sided marketplace. In this paper, we propose \texttt{CP-FairRank}, an optimization-based re-ranking algorithm that seamlessly integrates fairness constraints from both the consumer and producer side in a joint objective framework. The framework is generalizable and may take into account varied fairness settings based on group segmentation, recommendation model selection, and domain, which is one of its key characteristics. For instance, we demonstrate that the system may jointly increase consumer and producer fairness when (un)protected consumer groups are defined on the basis of their~\textit{activity level} and~\textit{main-streamness}, while producer groups are defined according to their popularity level. For empirical validation, through large-scale on eight datasets and four mainstream collaborative filtering (CF) recommendation models, we demonstrate that our proposed strategy is able to improve both consumer and producer fairness without compromising or very little overall recommendation quality, demonstrating the role algorithms may play in avoiding data biases. Our results on different group segmentation also indicate that the amount of improvement can vary and is dependent on group segmentation, indicating that the amount of bias produced and how much the algorithm can improve it depend on the protected group definition, a factor that, to our knowledge, has not been examined in great depth in previous studies but rather is highlighted by the results discovered in this study.
\end{abstract}



\received{20 February 2007}
\received[revised]{12 March 2009}
\received[accepted]{5 June 2009}

\maketitle

\section{Introduction}
\label{sec:introduction}

\begin{figure*}
  \centering

  \subfloat[Percentage of research]
  {
      \includegraphics[scale=0.5]{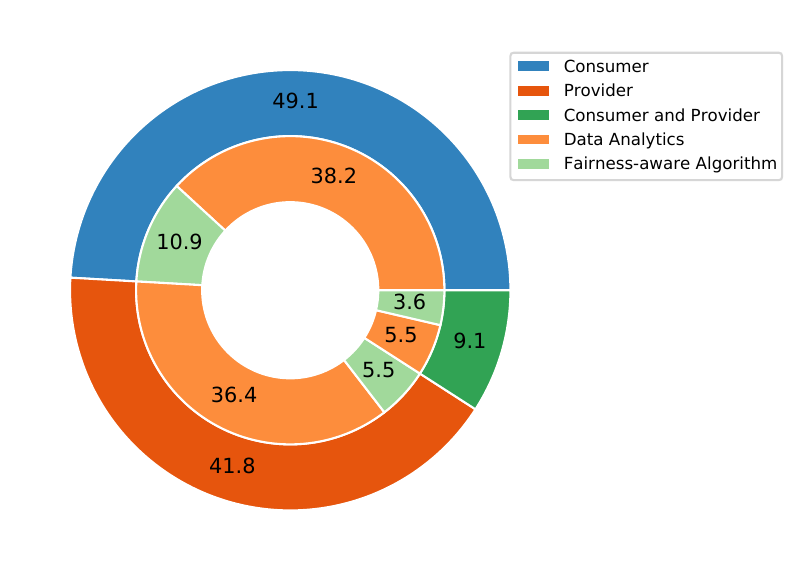}
      \label{fig:paperstatistics}
      
  }
  \hfill
    \subfloat[Evaluation Performance]
  {
      \includegraphics[scale=0.5]{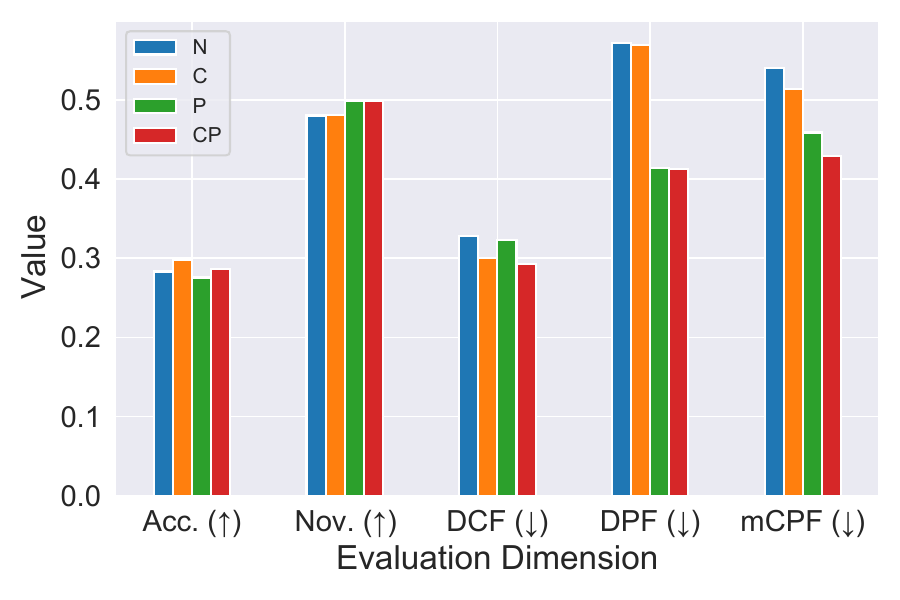}
      \label{fig:introresults}
  }

\caption{(a) The percentage of the research studied different aspects of fairness in recommender systems. (b) \updatednew{N (Fairness-unaware), C (user-oriented), P (item-oriented), and CP (two-sided fairness-performance) of recommendation algorithms on the accuracy, novelty, and fairness performance.} Consumer fairness evaluation (DCF) and producer fairness evaluation (DPF) are the two metrics that make up the mCPF. Note that CP represents the core of our contribution. The results on each bar show the average of 32 experiments across datasets and baseline CF models. \updatednew{$\uparrow$ means the higher the better while $\downarrow$ means the lower the better. Acc., Nov., and mCPF refer to accuracy, novelty, and consumer-producer fairness evaluation, respectively.}}
\label{fig:mm}
\end{figure*}

Recommender systems are ubiquitous and support high-stakes decisions in a variety of application contexts such as online marketing, music discovery, and high-stake tasks (\eg job search, candidate screening, and loan applications). By delivering personalized suggestions, these systems have a major impact on the content we consume online, our beliefs and decisions, and the success of businesses. Recently, there has been a growing awareness about the fairness of machine learning (ML) models in automated decision-making tasks, such as classification \cite{mehrabi2021survey,binns2018fairness,salutari2023quantifying,naghiaei2022towards}, and ranking/filtering tasks \cite{chen2020bias,caton2020fairness}, with recommender systems serving as a notable example of the latter. However, unlike classification, fairness in recommender systems (RS) is a multi-faceted subject, depending on stakeholder, type of benefit, context, morality, and time among others~\cite{ekstrand2021fairness,deldjoo2021flexible,burke2017multisided,wang2023improving,rahmani2022role}.

A recurring theme in the fairness recommendation literature is that computing research is often focused on a particular aspect of fairness, framing the problem as a building algorithm that fulfills a specific criterion, such as optimizing for producer fairness~\cite{dong2020user,zhu2021popularity,yalcin2021investigating} or consumer fairness~\cite{dwork2012fairness,islam2021debiasing}. However, this strategy may appear oversimplified, requiring more nuanced and multidisciplinary approaches to research recommendation fairness.

One overarching aspect that may help to unify studies on fairness in multi-stakeholder settings is distinguishing the \textit{(i)} benefit type (exposure vs.~relevance), and (\textit{ii)} the main stakeholders involved in the recommendation setting (consumers vs.~producers). Exposure refers to how evenly items or groups of items are exposed to users or groups of users. Relevance (effectiveness) determines to what extent the items exposition is effective, \ie matches with the user's taste. The second level of the taxonomy is stakeholder. Almost every online platform we interact with (\eg Spotify, Amazon) serves as a marketplace connecting consumers and item producers/service providers, making them the primary end-beneficiaries and stakeholders in RS. From the consumers' perspective, fairness is about distributing effectiveness evenly among users, whereas producers and item providers seeking increased visibility are primarily concerned with exposure fairness. 

Fig.~\ref{fig:paperstatistics} illustrates the distribution of research on various aspects of fairness in recommender systems.\footnote{The figures are based on 120 publications retrieved from DBLP using the keywords ``fair/biased recommendation'', ``re-ranking'', and ``collaborative filtering''.} We may observe a division/split in the research on fairness-aware recommendation algorithms, with around 49.1\% of articles focusing on consumer fairness and somewhat fewer on producer fairness (41.8\%). Few studies (less than 10\%) address consumer and producer fairness concerns simultaneously. However, the underlying user-item interaction coexists on (and can impact) both sides of beneficiary stakeholders. For example, there may be disparities in the consumption of item groups (defined by protected attributes) between active and inactive users in specific domains. Active users of Point-of-Interest (POI) may want to visit a variety of popular and less popular locations, whilst less active users may want to visit mainly the most popular locations. Prior research on unilateral fairness optimization raises various questions, including whether optimizing for one party's advantages, such as consumers, may ensure/protect the interests of other parties (\ie producers) and vice versa? Furthermore, how does optimizing for one stakeholder's benefits (\eg consumers' relevance) affect overall system accuracy? Are there ways to explain the key interplay and trade-offs between fairness restrictions (user relevance, item exposure) and overall system accuracy in various recommendation scenarios? Are there ways to explain the key interplay and trade-offs between fairness restrictions (user relevance, item exposure) and overall system accuracy in various recommendation scenarios?

\updatednew{In response to the questions raised}, this study concentrates on fairness within a \textit{multi-stakeholder} setting. Our goal is to harmonize two facets of fairness: fairness towards consumers and providers. To achieve this, we introduce a model-agnostic re-ranking methodology. This approach is designed to optimize marketplace objectives effectively. It ensures equitable exposure for producers' items and minimizes unequal treatment of different consumer groups, all while maintaining the overall accuracy of the recommender system.

To better illustrate the capabilities of our proposed system, in Fig.~\ref{fig:introresults}, we compare the performance of various recommendation algorithms against unilateral (or single-sided) fairness objectives according to the perspectives of different beneficiary stakeholders.

\updated{In particular, we consider four distinct variants of recommendation algorithms. The first, denoted as N, is a baseline recommendation algorithm that operates without any explicit fairness considerations. The subsequent three, labeled as C, P, and CP, are specific adaptations of our proposed methodology, each addressing a unique fairness objective. These objectives focus on consumer-side fairness in the case of C, provider-side fairness for P, and both objectives simultaneously for CP.
The C variant is closely modeled on the approach put forth by~\citet{li2021user}, albeit with a key modification in our rendition: the parity objectives of users are incorporated directly into the system's framework as an unconstrained optimization problem. To yield an optimal solution within a polynomial timeframe, we have chosen to employ a suitably efficient greedy algorithm.}

Fig.~\ref{fig:introresults} represents the average results of the 128 experimental cases examined in this work, \ie $8 \ \text{(datasets)} \ \times \ 4 \ \text{(CF baselines)} \ \times \ 4 \ (\text{fairness constraints type})$. One may observe the drawbacks of unilateral fairness optimization, in which DCF and DPF are used to measure unfairness (biases) on the consumer and supplier sides, respectively (cf.~Section~\ref{subsec:market_obj}). For example, user-oriented fairness optimization (C-fairness) can enhance both consumer fairness (\ie reducing DCF) and total system accuracy. However, this creates a huge bias against suppliers. Producer-fairness optimization, on the other hand, can improve novelty and producer fairness (\ie reducing DPF), indicating that users' recommendation lists contain a greater number of long-tail items. However, this comes at the cost of a considerable reduction in overall accuracy (from $0.2832$ to $0.2756$) and worsening of DCF (from $0.3280$ to $0.3234$). 
The proposed consumer-provider fairness optimization (CP-fairness) in the present work combines the benefits of previous system variants, enhancing both producer and consumer fairness, resulting in increased accuracy (from $0.2832$ to $0.2868$) and even increased novelty (due to recommendation of more long-tail items). To summarize, the key contributions of this work are as follows:
\begin{itemize}
    \item \textbf{Motivating multi-stakeholder fairness in RS.} We motivate the importance of multi-sided fairness optimization by demonstrating how inherent biases in underlying data may impact both consumer and producer fairness constraints, as well as the overall system's objectives, \ie accuracy and novelty (cf.~Section~\ref{sec:motivating_cp}); these biases and imbalances if left unchecked, could lead to stereotypes, polarization of political opinions, and the loss of emerging business. 
    \item \textbf{CP-Fairness modeling.} Our research builds on and expands earlier work \cite{li2021user}, in several dimensions: \textit{(i)} we consider \textit{multi-stakeholder} objectives including both consumers' and producers' fairness aspects; \textit{(ii)} we formalize the re-ranking problem as an integer programming method without enforcing the fairness objectives as a constraint that is susceptible to reduction in overall accuracy in a \textit{marketplace}; \textit{(iii)} we propose an efficient greedy algorithm capable of achieving an optimal trade-off within \textit{multi-stakeholder} objectives in a polynomial-time (cf.~Section~\ref{sec:proposed_method}). In summary, this makes the proposed contribution more versatile and applicable to various recommendation models and settings irrespective of their learning criteria and their recommendation list size.
    \item \textbf{Experiments.} We conduct extensive experimental evaluations on 8 real-world datasets from diverse domains (\eg Music, Movies, E-commerce), utilizing implicit and explicit feedback preferences (cf.~Section~\ref{sec:experimental_methodology}); this, combined with four recommendation algorithm baseline amount the number of experiments carried out to 128 recommendation simulation.
\end{itemize}

\noindent \updatednew{\textbf{Contributions and enhancements beyond the original version of this work.}
This work builds upon our previous research presented in SIGIR 2022 \cite{naghiaei2022cpfair}, and extends it in the following ways:}

\begin{enumerate}
    \item \updatednew{\textbf{Revisiting User Group Definitions in Recommender Systems:} Building on our previous research in SIGIR 2022 \cite{naghiaei2022cpfair}, this work introduces a novel perspective on the grouping of the users, based on the quality of user activity in terms of product popularity. We put forth the notion that users consuming more mainstream popular products rather than a diverse range can influence collaborative filtering (CF) models to favor popular items.  For instance, if a user predominantly interacts with mainstream, popular products, this preference could inadvertently bias collaborative filtering (CF) models towards suggesting more popular items. This introduces a new dimension of unfairness in recommendation outcomes, which may vary in magnitude and proportion depending on the number and type of activities a user engages in. Consequently, we have reanalyzed all experiments from our SIGIR'22 submission, introducing new fairness scenarios that differentiate between mainstream and non-mainstream users. The objective is to determine if these groups are provided with better recommendation quality and to evaluate if our methods can effectively enhance fairness in these scenarios.}

    \item \updatednew{\textbf{Enhanced Generalizability and New Statistical Insights:} To increase the applicability of our findings, additional analyses have been incorporated, including correlation plots and fairness-accuracy trade-off evaluations for two distinct user groups (refer to Table \ref{tbl:correlation} and Table \ref{tbl:tradeoff}). These new statistical measures provide a better representation of the generality of our findings across various domains.}
\end{enumerate}

\section{Background and Related work}
\label{sec:related_work}
A promising approach to classifying the literature in recommendation fairness is according to the beneficiary stakeholder~\cite{abdollahpouri2020multistakeholder,ekstrand2021fairness,schedl2023fairness}. Based on a recent survey screening results on papers published in fairness during the past years~\cite{deldjoo2023fairness}, only 9.1\% of publications in the field deal with multi-sided fairness, with even a less percentage of $3.6\%$ proposing an actual two-sided fairness-aware algorithm, as pursued in our current study.

\partitle{C-fairness Methods.}
\citet{dwork2012fairness} introduces a framework for individual fairness by including all users in the protected group. \citet{abdollahpouri2019unfairness} and \citet{naghiaei2022unfairness} investigate a user-centered evaluation of popularity bias that accounts for different levels of interest among users toward popular items in movies and book domain. \citet{abdollahpouri2021user} propose a regularization-based framework to mitigate this bias from the user perspective. \citet{li2021user} address the C-fairness problem in e-commerce recommendation from a group fairness perspective, \ie a requirement that protected groups should be treated similarly to the advantaged group or total population \cite{pedreschi2009measuring}. Several recent studies have indicated that an exclusively consumer-centric design approach, in which customers' satisfaction is prioritized over producers' interests, may result in a reduction of the system's overall utility~\cite{patro2020fairrec,wang2021user}. As a result, we include the producer fairness perspectives in our proposed framework as the second objective in this work.

\partitle{P-fairness Methods.}
Several works have studied recommendation fairness from the producer's perspective \cite{yalcin2021investigating,wundervald2021cluster}. \citet{gomez2022provider} assess recommender system algorithms disparate exposure based on producers' continent of production in movie and book recommendation domain and propose an equity-based approach to regulate the exposure of items produced in a continent. \citet{dong2020user} set a constraint to limit the maximum times an item can be recommended among all users proportional to its popularity to enhance the item exposure fairness. In contrast, \citet{ge2021towards} investigate fairness in a dynamic setting where item popularity changes over time in the recommendation process and models the recommendation problem as a constrained Markov Decision Process with a dynamic policy toward fairness.

A common observation in C-fairness and P-Fairness research is the type of attributes used to segment the groups on the consumer and provider side. These attributes could be internal, e.g., determined via user-item interactions~\cite{li2021user,abdollahpouri2019managing,karimi2023provider} or provided externally, e.g., protected attributes such as gender, age, or geographical location~\cite{gomez2022provider,boratto2021interplay}. In this work, we segmented the groups by internal attributes and based on the number of interactions for both consumers and providers.

\partitle{CP-fairness Methods.}
Considering both consumers' and providers' perspectives, \citet{chakraborty2017fair} present mechanisms for CP-fairness in matching platforms such as Airbnb and Uber. \citet{rahmani2022unfairness} studied the interplays and tradeoffs between consumer and producer fairness in Point-of-Interest recommendations. \newtxt{\citet{patro2020fairrec} map fair recommendation problem to the constrained version of the fair allocation problem with indivisible goods and propose an algorithm to recommend top-K items by accounting for producer fairness aspects.} \citet{wu2021tfrom} focus on two-sided fairness from an individual-based perspective, where fairness is defined as the same exposure to all producers and the same \updatednew{normalized Discounted Cumulative Gain (nDCG)} to all consumers involved.

\citet{do2021two} define the notion of fairness in increasing the utility of the worse-off individuals following the concept of distributive justice and proposed an algorithm based on maximizing concave welfare functions using the Frank-Wolfe algorithm. \citet{lin2021mitigating} investigate sentiment bias, the bias where recommendation models provide more relevant recommendations on user/item groups with more positive feedback, and its effect on both consumers and producers.
Numerous research studies have also concentrated on assessing both consumer and producer fairness \cite{deldjoo2021flexible,anelli2023auditing}, as exemplified by \cite{anelli2023auditing} for graph CF models.

The inner layer of the circle in Fig.~\ref{fig:paperstatistics} presents approaches for recommendation fairness that could be classified according to 1) developing metrics to quantify fairness, 2) developing frameworks for producing fair models (according to the desired notion of fairness). Contrary to the works surveyed, we develop a fairness-aware algorithm that simultaneously addresses user groups and item group fairness objectives jointly via an efficient optimization algorithm capable of increasing the system's overall accuracy.
\section{Motivating CP-Fairness Concerns}
\label{sec:motivating_cp}

\begin{table}
    \centering
    \caption{Percentage of users and items located at a different number of interaction thresholds (as $n$ and $r$ represents, respectively) in the training set of the datasets.}
    \begin{adjustbox}{max width=340pt}
    \begin{tabular}{l|llll|llll}
    \toprule
        Dataset & $n\geq10$ & $n\geq20$ & $n\geq50$ & $n\geq100$ & $r\geq10$ & $r\geq20$ & $r\geq50$ & $r\geq100$  \\ \midrule
        MovieLens & 100\% & 82.08\% & 46.98\% & 26.09\% & 77.39\% & 60.49\% & 34.54\% & 16.46\% \\
        Epinion & 99.85\% & 56.60\% & 8.11\% & 1.57\% & 99.81\% & 63.83\% & 17.33\% & 4.61\% \\
        Gowalla & 100\% & 86.11\% & 25.40\% & 4.87\% & 100\% & 84.44\% & 20.10\% & 3.95\% \\
        LastFM & 96.33\% & 72,73\% & 0.00\% & 0.00\% & 73.99\% & 37.36\% & 15.00\% & 4.48\%\\
    \bottomrule
    \end{tabular}
    \label{tbl:dataset_charactristics}
    \end{adjustbox}
\end{table}

In this section, we intend to motivate the need of having two-sided fairness (\ie user and item) by undertaking both data and algorithm analysis. We analyze the distribution and properties of two well-known real-world recommendation datasets, \ie \datasetB and \datasetC, with their details summarized in Table~\ref{tbl:datasets}. We observed similar patterns on the other datasets as shown in Fig.~\ref{fig:CPevalBoxPlotUG1} baseline (N) models.

In Table~\ref{tbl:dataset_charactristics}, we show the distribution of users and items with different numbers of interactions in the datasets. As the values show, most users are concentrated in areas with less interaction with the items. We can also note that the majority of items have received fewer interactions from the users (\ie long-tail items), and a small percentage of items (\ie popular items) are recommended frequently. This motivates the need to expand the long-tail coverage of recommendation lists while encouraging users with fewer interactions to engage with the system.

\begin{figure*}
  \centering
  \subfloat[MovieLens]
    {\includegraphics[scale=0.2]{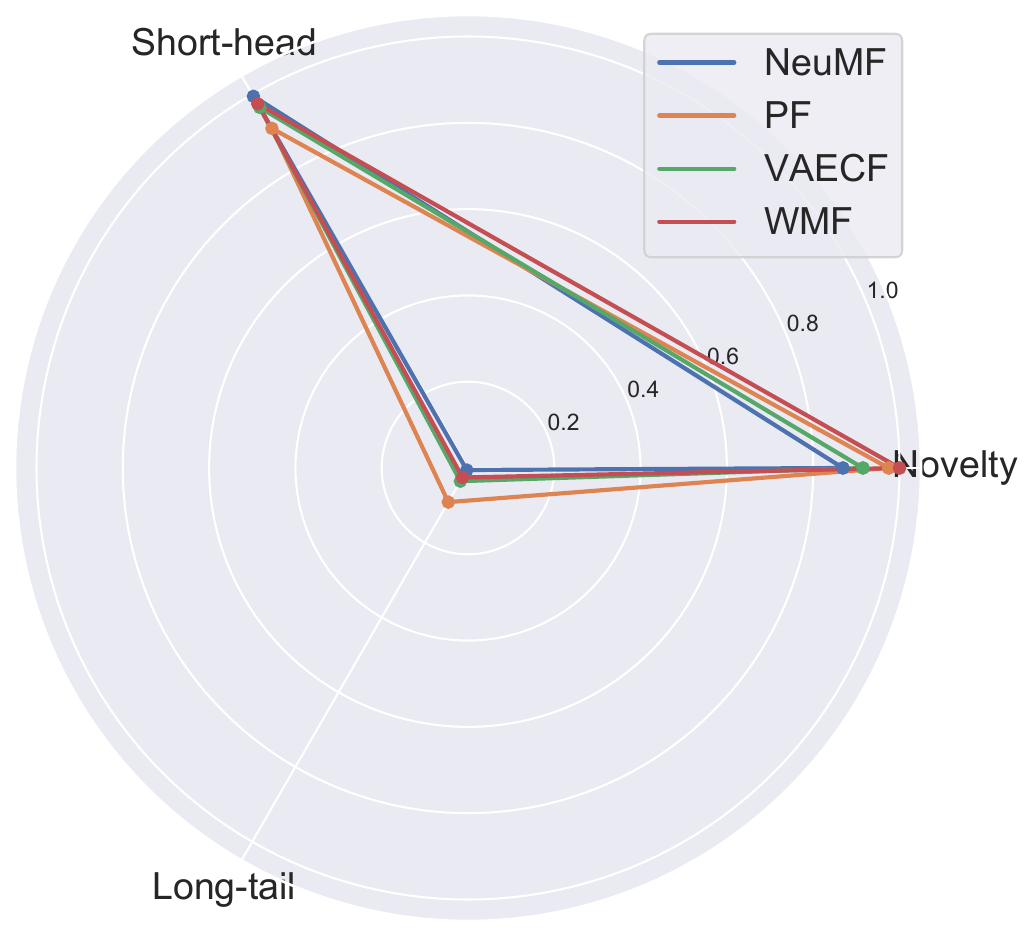}
    \label{fig:radar_item_movielens}}
  \hfill
  \subfloat[Epinion]
    {\includegraphics[scale=0.2]{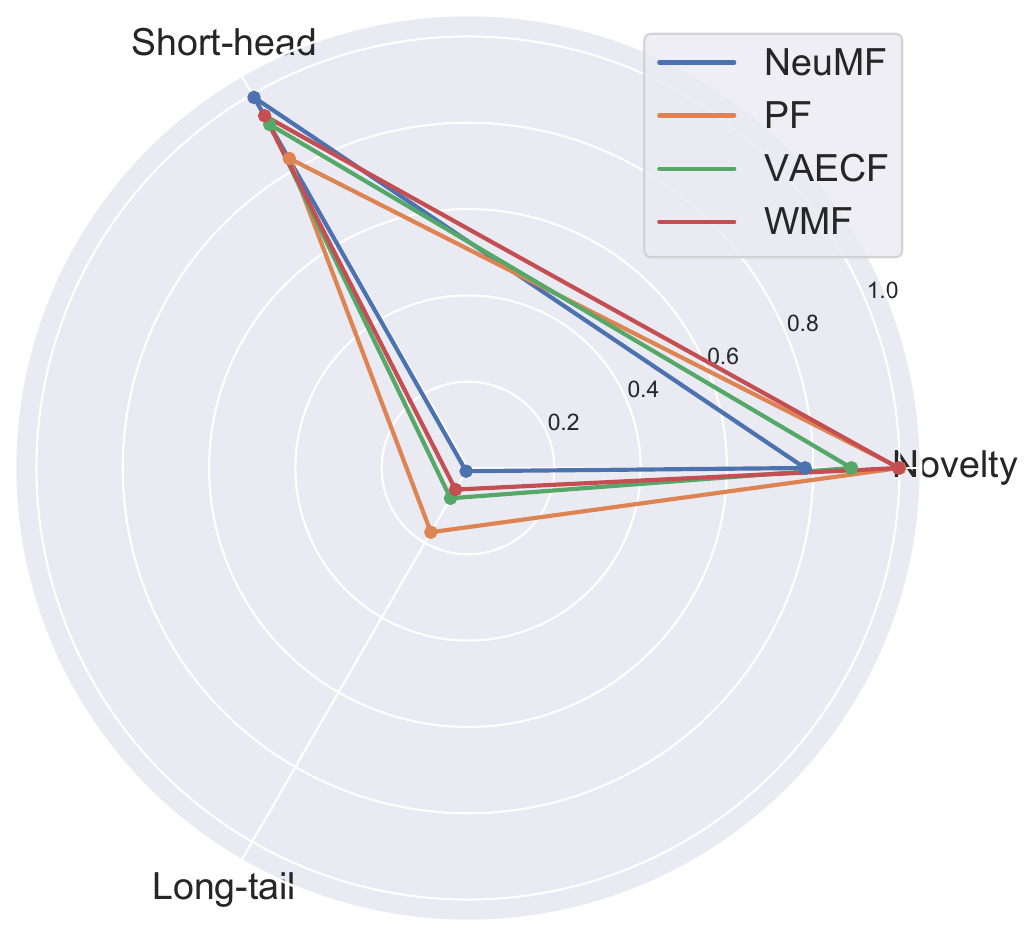}
    \label{fig:radar_items_epinion}}
  \hfill
  \subfloat[LastFM]
    {\includegraphics[scale=0.2]{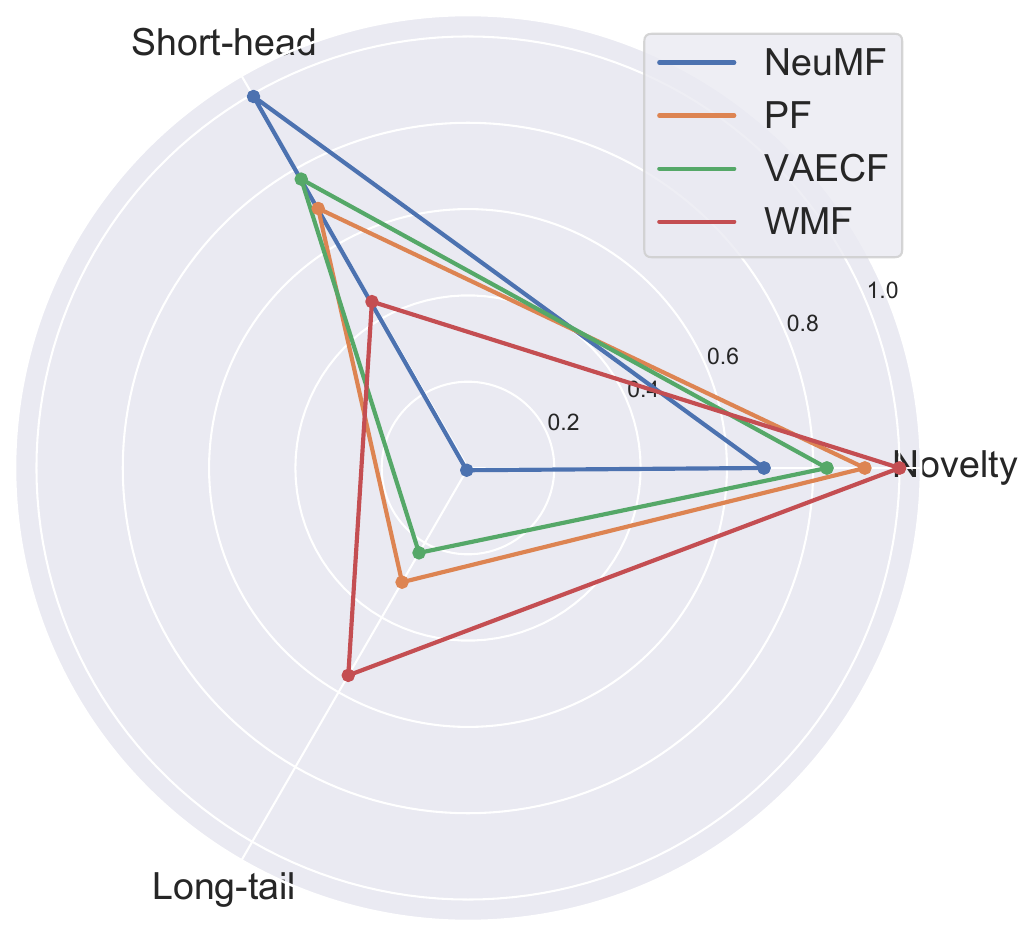}
    \label{fig:radar_item_lastfm}}
  \hfill
  \subfloat[BookCrossing]
    {\includegraphics[scale=0.2]{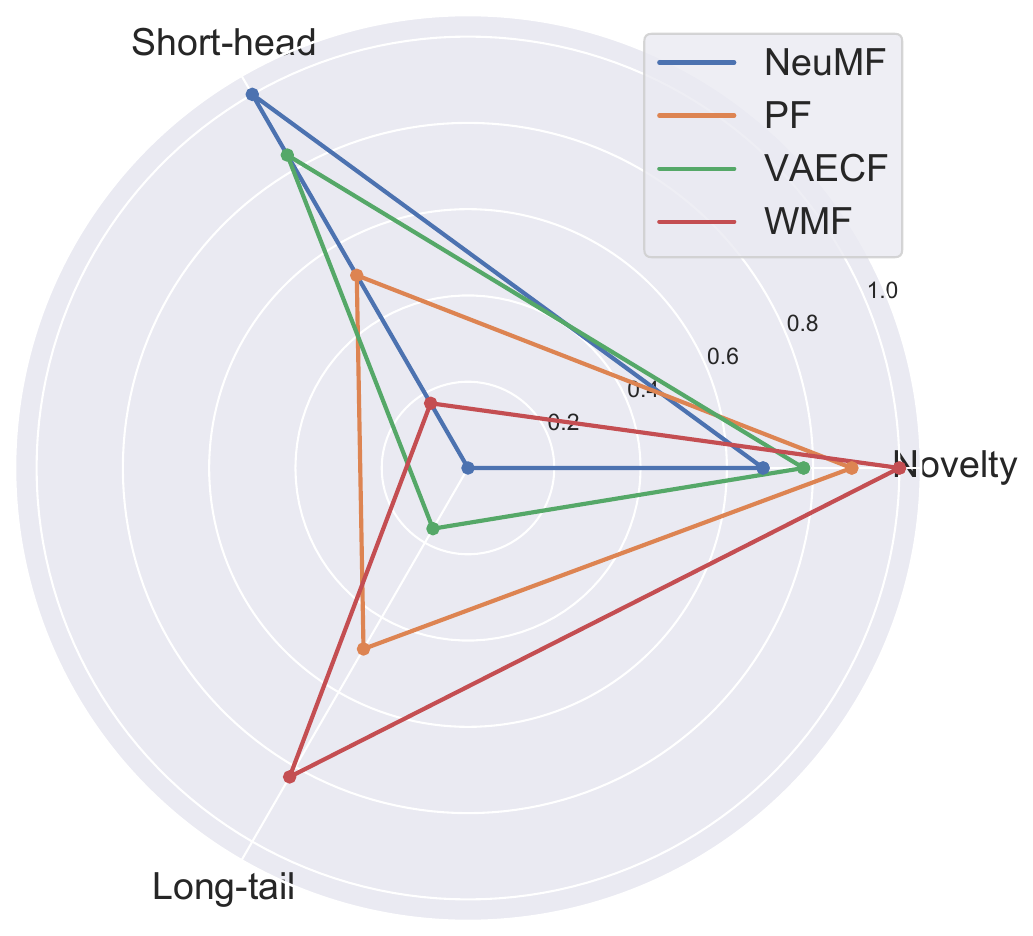}
    \label{fig:radar_item_bookcrossing}}
  \caption{The difference between \itemgroupA and \itemgroupB item groups on item exposure and novelty.}
\label{fig:item_unfairness_rec}
\end{figure*}

\begin{figure*}
    \centering
  \subfloat[MovieLens]
    {\includegraphics[scale=0.2]{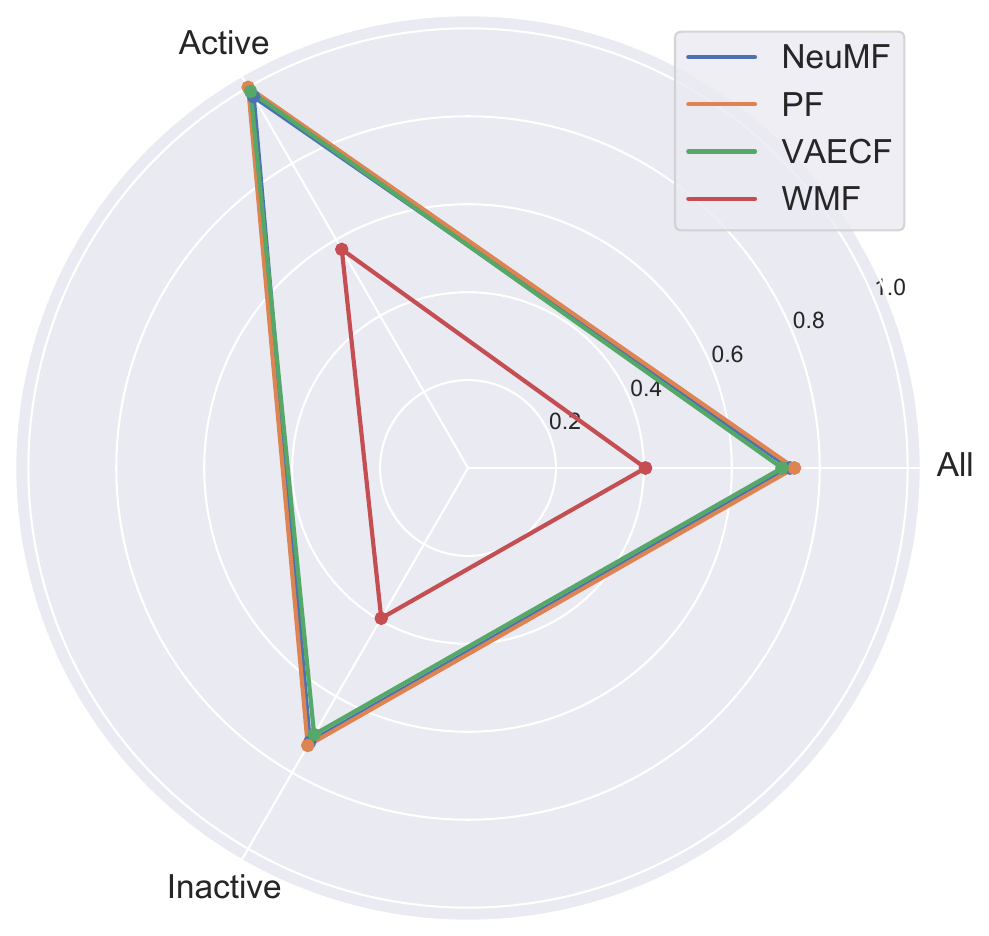}
    \label{fig:radar1_user_movielens}}
  \hfill
  \subfloat[Epinion]
    {\includegraphics[scale=0.2]{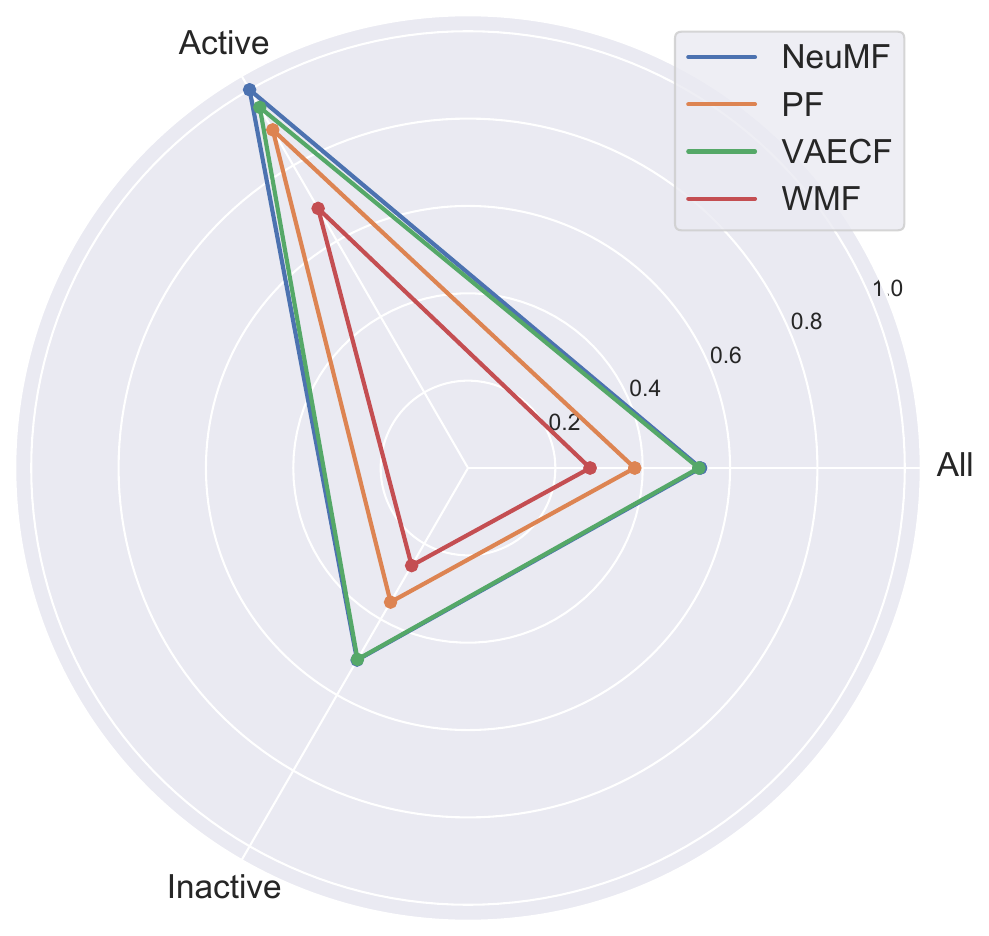}
    \label{fig:radar1_user_epinion}}
  \hfill
  \subfloat[LastFM]
    {\includegraphics[scale=0.2]{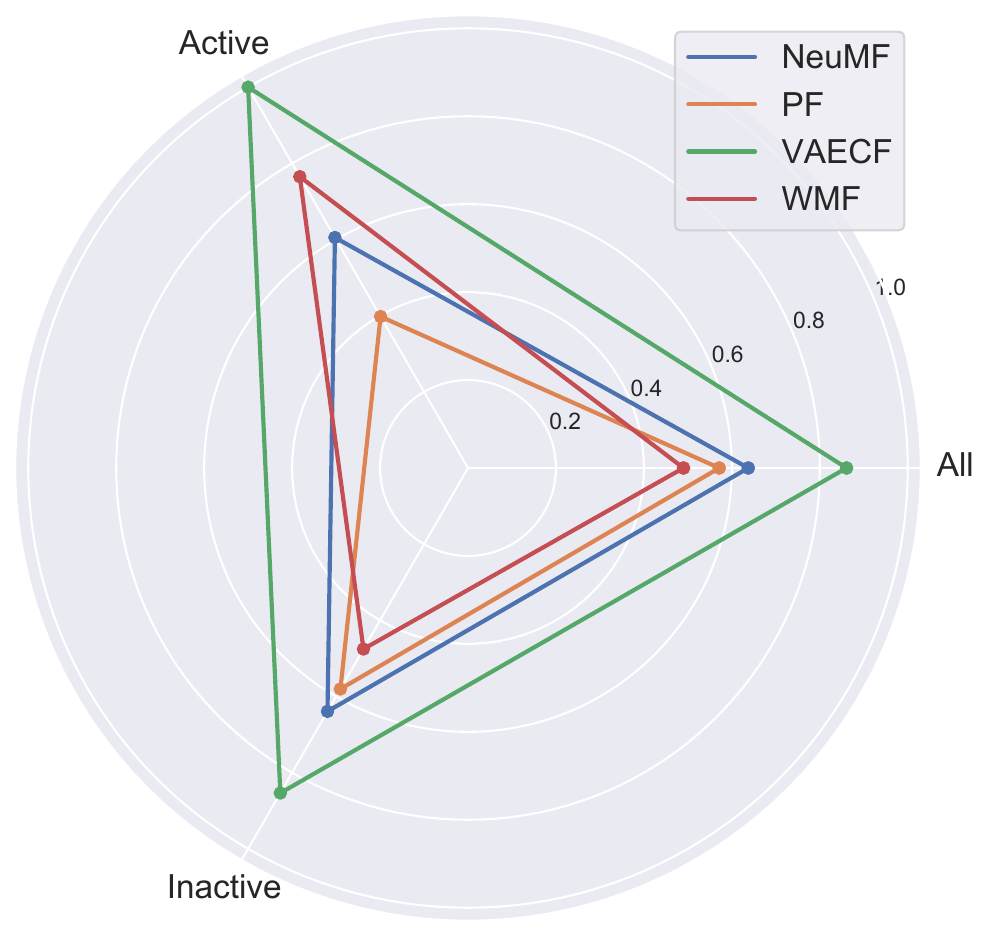}
    \label{fig:radar1_user_lstfm}}
  \hfill
  \subfloat[BookCorssing]
    {\includegraphics[scale=0.2]{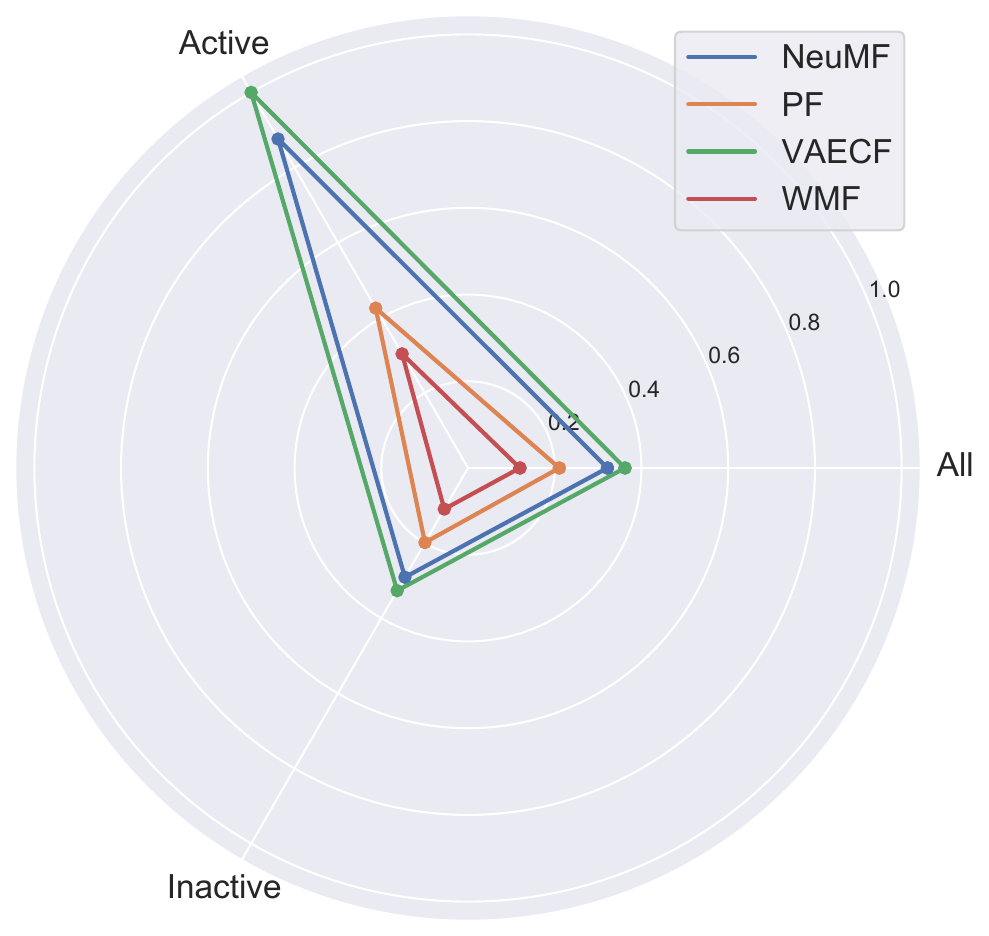}
    \label{fig:radar1_user_bookcrossing}}
    \hfill
\subfloat[MovieLens]
    {\includegraphics[scale=0.2]{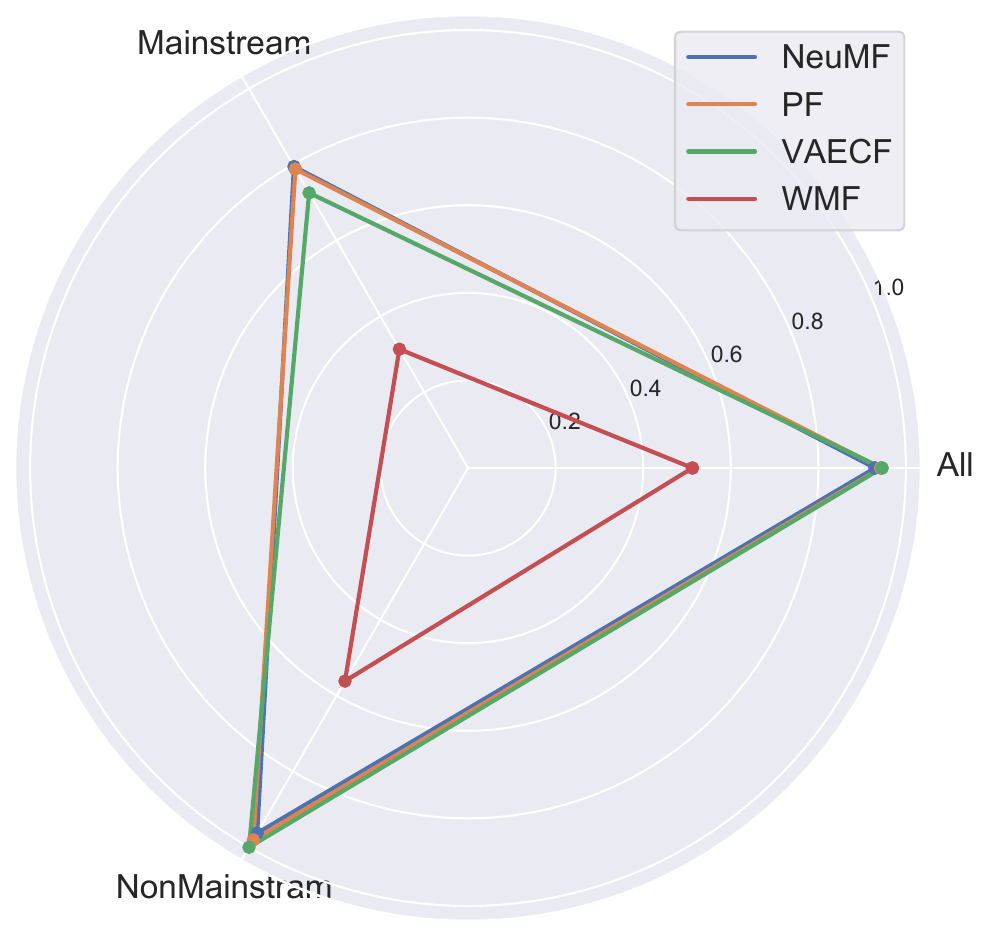}
    \label{fig:radar2_user_movielens}}
  \hfill
  \subfloat[Epinion]
    {\includegraphics[scale=0.2]{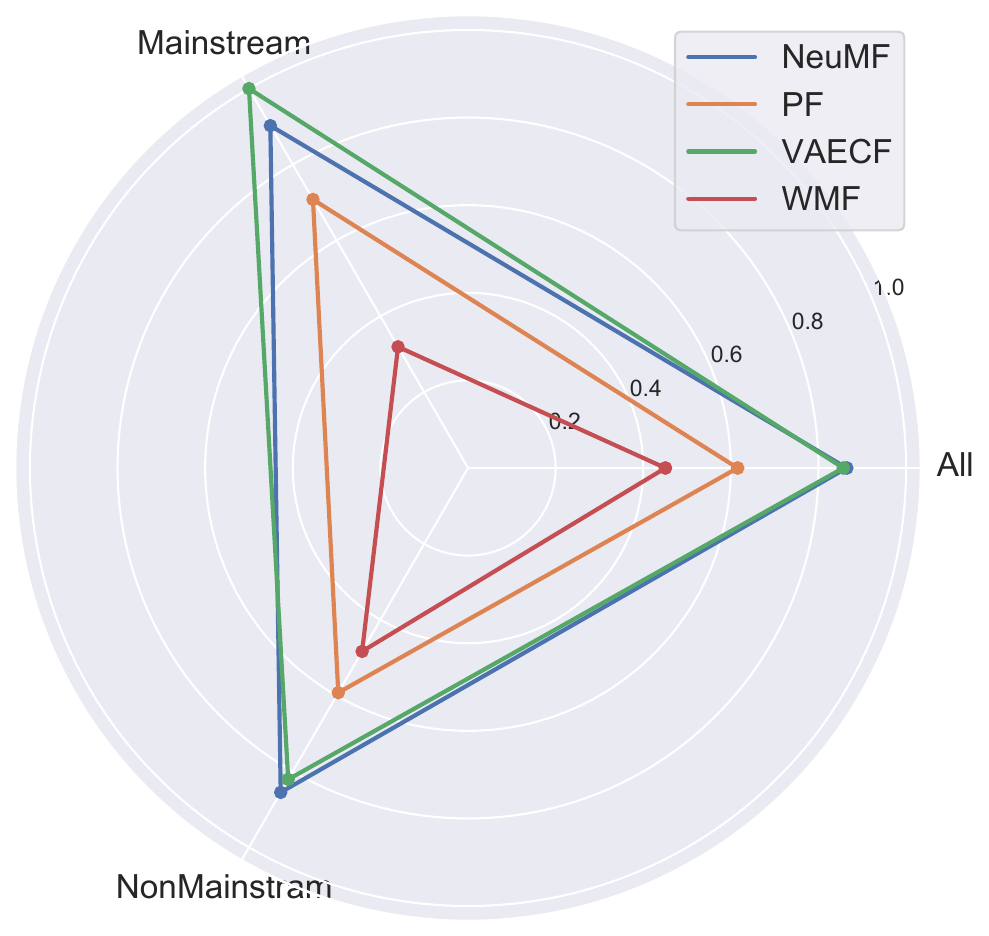}
    \label{fig:radar2_user_epinion}}
  \hfill
  \subfloat[LastFM]
    {\includegraphics[scale=0.2]{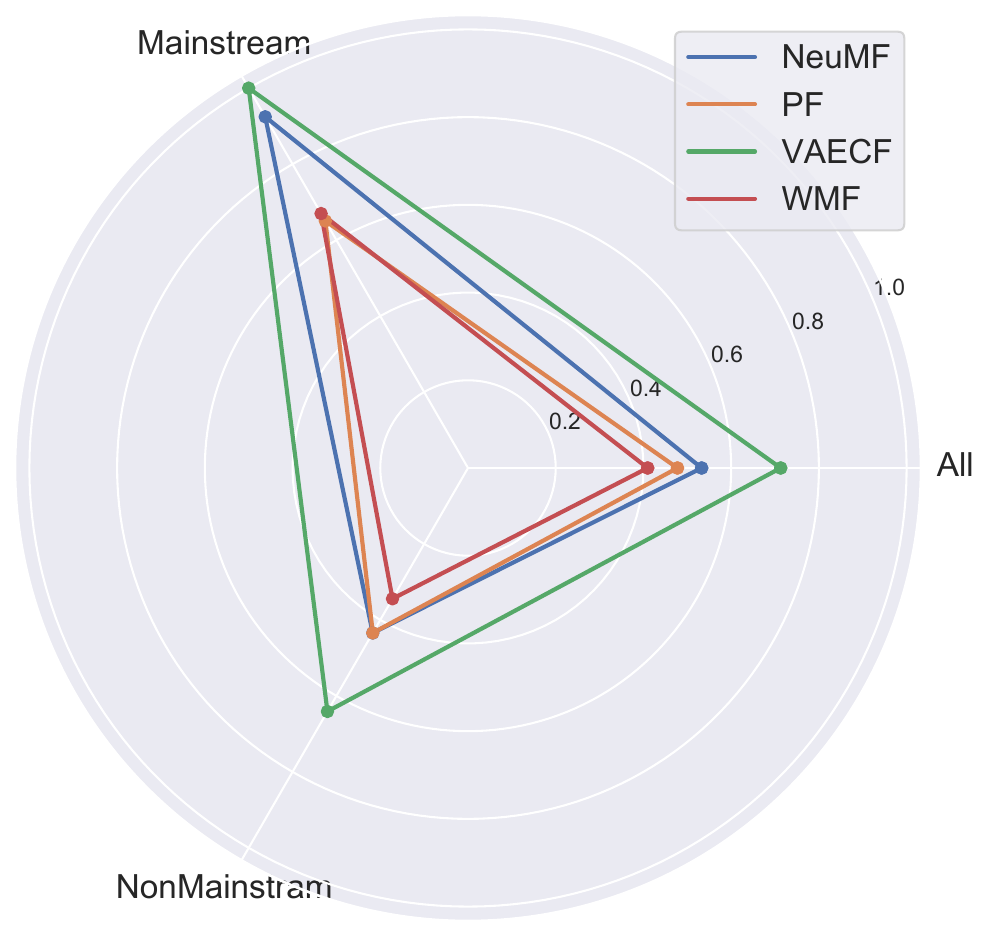}
    \label{fig:radar2_user_lstfm}}
  \hfill
  \subfloat[BookCorssing]
    {\includegraphics[scale=0.2]{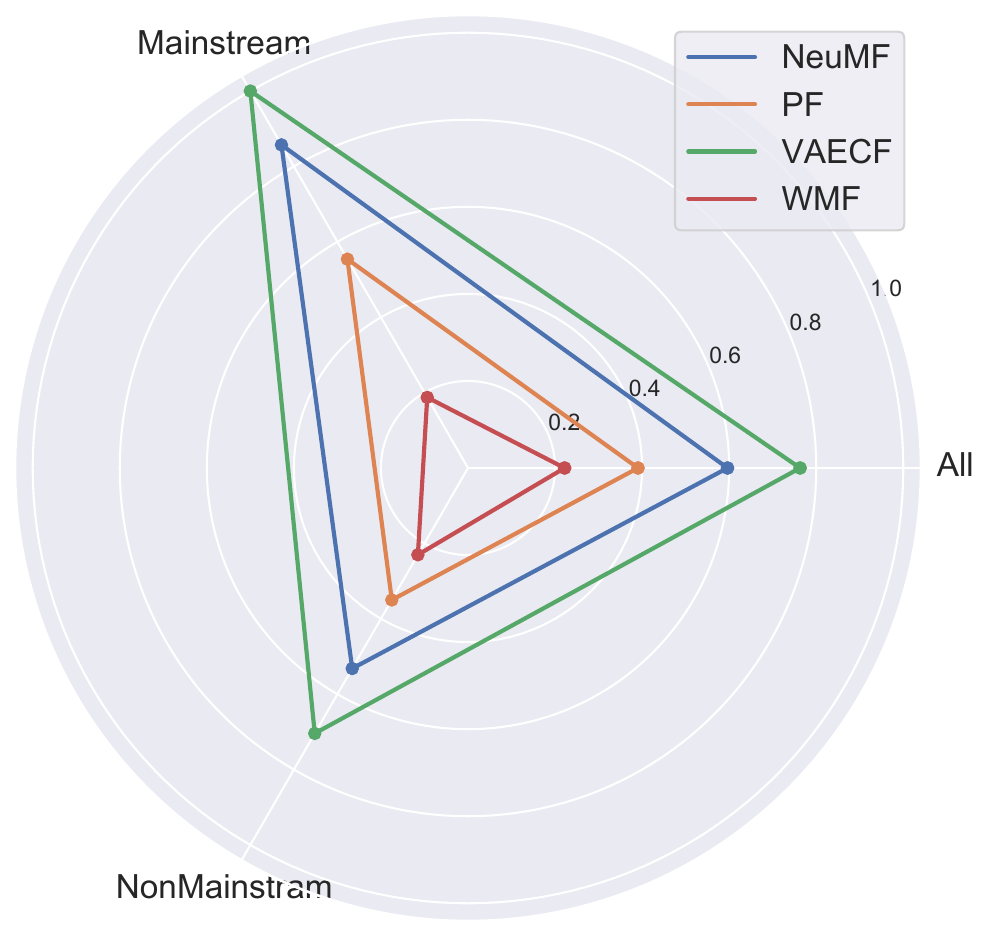}
    \label{fig:radar2_user_bookcrossing}}
  \caption{Figures (a) and (d) show the difference between \usergroupA and \usergroupB user groups and Figures (e) and (d) indicate the difference between mainstream and non-mainstream user groups on normalized nDCG@10.}
\label{fig:user_unfairness_rec}
\end{figure*}

In our next step, we evaluate the recommendation quality of fairness-unaware baseline algorithms on two datasets for both user and item groups. Figs.~\ref{fig:item_unfairness_rec} and \ref{fig:user_unfairness_rec} showcase the min-max normalized values of the recommendation algorithms' quality across the MovieLens, Epinion, LastFM, and BookCrossing datasets. Fig.~\ref{fig:item_unfairness_rec} plots the number of short-head and long-tail items recommended to users, as well as the novelty of the recommendations. Our results show that despite the small number of items in the short-head group compared to the long-tail group, a large proportion of recommended items are popular items, implying a higher degree of unfairness in the recommendation from the provider's perspective (measured in terms of exposure). Moreover, Fig.~\ref{fig:user_unfairness_rec} plots the accuracy (i.e., nDCG@10) for all users, active users, inactive users, mainstream users, and non-mainstream user groups. In Section \ref{sec:results}, we present a detailed analysis of the experiment results in Table \ref{tbl:resultsseg1}. Our findings reveal that the active user group, which represents a small fraction of total users, receives significantly higher recommendation quality than the inactive user group (similarly for mainstream and non-mainstream). This demonstrates that commercial recommendation engines tend to overlook the majority of users. For instance, from Fig.~\ref{fig:radar1_user_epinion}, we observe that NeuMF and VAECF achieve almost the best performance (1.0) in the active user group, while for the inactive user group, the performance is around 0.6. This degradation in overall performance, which is only slightly above the inactive group at 0.6 for both recommendation algorithms, highlights the need for fairness in recommendations. In conclusion, our findings highlight the importance of fairness in recommendations from both the consumers' and providers' perspectives, and thus developing techniques that better serve disadvantaged users and item groups is necessary.
\section{Proposed CP-fairness algorithm}
\label{sec:proposed_method}
In this section, we first define fairness in a multi-sided platform based on deviation from the parity from consumers' and producers' perspectives. Then we propose a mixed-integer linear programming re-ranking algorithm to generate a fair list of recommendations for each user given a traditional unfair recommendation list.

\partitle{Problem Formulation.}
Let $\mathcal{U}$ and $\mathcal{I}$ be the set of consumers and items with size $n$ and $m$, respectively. In a traditional recommender system application, a user, $u \in \mathcal{U}$, is provided a list of top-N items, $L_N(u)$ according to some relevance scores, $s \in S^{n\times N}$. The goal is to compute a new fair ordered list, $L^{F}_K(u)$ from the original unfair baseline recommended items $L_N(u)$ where $L^{F}_K(u) \subset L_N(u) \ \ where \ \ K \leq N$ such that the new list would be more appealing to the system as a whole considering the relevance score to the user $u$ and fairness objectives jointly.

Formally, item $I_x$ is presented ahead of item $I_y$ in $L_N(u)$ if $s_{ux} > s_{uy}$. Recommending the items to users based solely on the relevance score is designed to help improve the overall system recommendation accuracy regardless of the fairness measures in a multi-sided platform. We define a binary decision matrix $A=[A_{ui}]_{n \times N}$ with value 1 when the item ranked $i$ in $L_N(u)$ is recommended to user $u$ in the new fairness-aware recommendation list, i.e., $L^F_K(u)$ and 0 otherwise. Hence, the fair top-K recommendation list to user $u$, can be represented as $[A_{u1},A_{u2},...,A_{uN}]$ when $\sum_{i=1}^{N}{A_{ui}} = K $, $K\leq N$.
We discuss below an  optimization framework to calculate optimal values for matrix $A$ and re-rank original unfair recommendation list.

\subsection{Marketplace Fairness Objectives}
\label{subsec:market_obj}
The fairness representation target, \ie the ideal distribution of exposure or relevance among groups, is the first step to designing a fair recommendation model that should be decided based on the application and by the system designers. \citet{kirnap2021estimation} classified the fairness notions commonly used in IR literature into three groups based on their ideal target representation, (1) parity between groups; (2) proportionality in exposure in groups according to the group/catalog size; and (3) proportionality to the relevance size to the recommended groups. This work focuses on the notion of fairness defined based on a \dquotes{parity-oriented} target representation, leaving other concepts of fairness for future investigations.

\subsubsection{Fairness constraint on producer side} 
On the producer side, fairness may be quantified in terms of two distinct types of benefit functions: exposure and relevance, both of which can be binary or graded \cite{gomez2022provider}. Given that suppliers' primary objective is their products' visibility, this study focuses its attention exclusively on the (binary) exposure dimension, leaving the other producer fairness constraints for future work. To measure the exposure parity, we divide items into two categories of short-head $\mathcal{I}_1$ and long-tail $\mathcal{I}_2$ defined based on their popularity score where $\mathcal{I}_1 \cap \mathcal{I}_2 = \emptyset $. 
\updated{We define the notation of $\mathcal{MP}$ as a metric that quantifies item groups' exposition in the recommendation list decision, $A$. In our experiments, we used the item group exposure count to quantify this.} Deviation from Producer Fairness (DPF) is the metric that defines the deviation from parity exposition defined by:

\begin{equation}
    \mathbf{DPF}(A, \mathcal{I}_1, \mathcal{I}_2) = \E\left[\mathcal{MP}(i,A)| i \in \mathcal{I}_1 \right] - \E\left[\mathcal{MP}(i,A)|i \in \mathcal{I}_2\right] 
    \label{eq:dpf}
\end{equation}

\subsubsection{Fairness constraint on the consumer side}
We divide users into two groups of frequent users, $\mathcal{U}_1$, and less-frequent users, $\mathcal{U}_2$ based on their activity level. \updated{The notation $\mathcal{MC}$ is a metric that can evaluate the recommendation relevance in each group given the recommendation list decision matrix, $A$. In our experiments, we assessed the relevance using group $nDCG$ calculated on the training dataset}. The ideal fair recommendation would recommend items based on the user's known preference regardless of their activity level and the quality of recommendation on frequent and less-frequent users should not get amplified by the recommendation algorithm. We define the deviation from consumer fairness parity as:

\begin{small}
\begin{equation}
\mathbf{DCF}(A, \mathcal{U}_1, \mathcal{U}_2) = \E\left[\mathcal{MC}(u,A) | u \in \mathcal{U}_1 \right] -\E\left[\mathcal{MC}(u,A) | u \in \mathcal{U}_2 \right]
    \label{eq:dcf}
\end{equation}
\end{small}

Note that when the values of $DPF\rightarrow 0$ or $DCF\rightarrow 0$, it shows parity between the advantaged and disadvantaged user and item groups. However, this may suggest a reduction in the overall system accuracy. The sign of $DCF$ and $DPF$ shows the direction of disparity among the two groups, which we decided to keep in this work. \newtxt{It is noteworthy that the proposed metrics $DPF$ and $DCF$, due to their linear nature with respect to the binary decision matrix $A$, present computational efficiency over metrics that necessitate non-linear optimization approaches, thereby simplifying the integration into existing recommendation system frameworks.} \\

\noindent \newtxt{\textbf{Note on Fairness Evaluation Metrics Selection.} Our study employs DCF and DPF metrics that encapsulate consumer effectiveness and producer exposure, aligning with fairness concepts prevalent in recommender systems literature \cite{deldjoo2023fairness}. While these metrics offer a focused lens on recommendation fairness, we acknowledge the broader ecosystem of fairness metrics and prioritize a deeper cross-metric analysis in future work to understand their interdependencies and collective impact on recommendation ecosystems.}

\subsection{Proposed Algorithm}
In this section, we will use the top-N recommendation list and relevance scores provided by traditional unfair recommendation algorithm as baseline. We apply a re-ranking post-hoc algorithm to maximize the total sum of relevance scores while minimizing the deviation from fairness on the producer and consumer perspective, $\mathbf{DCF}$ and $\mathbf{DPF}$ respectively. We can formulate the fairness aware optimization problem given the binary decision matrix, A, as:

\begin{small}
\begin{equation}
\begin{aligned}
\max_{A_{ui}} \quad & \sum_{u \in \mathcal{U}}\sum_{i=1}^{N}{S_{ui}.A_{ui}} - \lambda_1 \times \mathbf{DCF}(A,\mathcal{U}_1,\mathcal{U}_2) - \lambda_2 \times \mathbf{DPF}(A,\mathcal{I}_1,\mathcal{I}_2)\\
\textrm{s.t.} \quad & \sum_{i=1}^{N}{A_{ui}} = K, A_{ui} \in{\{0,1\}}
\end{aligned}
\label{eq:optimization}
\end{equation}
\end{small}

The solution to the above optimization problem would recommend exactly $K$ items to each user while maximizing the summation of preference scores and penalizing deviation from fairness. $ 0 \leq \lambda_1,\lambda_2 \leq 1$ are hyperparameters that would be selected based on the application and would assign how much weight is given to each fairness deviation concern. A complete user-oriented one-sided system could assign high weight to $\mathbf{DCF}$ and 0 to $\mathbf{DPF}$, and vice versa for a provider-oriented system. Note that if the value of $\lambda_1,\lambda_2$ are both equal to $0$ then the optimal solution will be the same as the baseline recommendation system list, that is, $L^{F}_K(u) = L_K(u)$. The effect of different values for $\lambda_1$ and $\lambda_2$ are studied in the ablation study section. Parity between groups is defined as the ideal representation of the fairness in the model but is not enforced as a constraint in our optimization to avoid resulting in diminishing the total model accuracy in a significant way. 

The optimization problem above is a mixed-integer linear programming (MILP) and a known NP-hard problem. Although MILP cannot be generally solved optimally in polynomial time, there exists many fast heuristic algorithms and industrial optimization solvers that would provide a good feasible solution in practice\footnote{\eg Gurobi solver at \url{https://www.gurobi.com}}. The MILP proposed optimization at Equation~\ref{eq:optimization} can be reduced to a special case of Knapsack problem where the goal is to select items for each user that would maximize the total score in which each item has a weight equal to 1 and the total weight is constrained by fixed size K. Since each item has the same weight in this instance of Knapsack problem, we can solve it optimally by a greedy algorithm in polynomial time. Now we are ready to propose the Fair Re-ranking Greedy algorithm (Algorithm~\ref{alg:LBP}) that would give the optimal solution for the optimization problem in Equation~\ref{eq:optimization}. Alternatively, we can relax the binary constraint of variables to fractional variables with boundary constraints \ie $0\leq A_{ui}\leq1$ and solve the relaxed optimization problem in polynomial time. The algorithm based on optimization is presented in Algorithm~\ref{alg:optmization}. \updated{The algorithm based on optimization, Algorithm~\ref{alg:optmization}, is designed to solve the relaxed optimization problem, converting the MILP problem to a standard linear programming problem. In this relaxed problem, we use a simplex algorithm, a standard method for numerically solving such linear programming problems. The simplex algorithm is an iterative method that navigates the corners of the feasible region defined by the problem constraints to find the optimal solution, which maximizes the objective function. This approach allows us to find the best solution that satisfies all the constraints in polynomial time, making it a practical choice for the type of problems we are addressing in this study. However, the solution obtained may be fractional and not necessarily integral, due to the relaxation of the binary constraints.}

\begin{algorithm}
	\DontPrintSemicolon
	\SetAlgoLined
	\SetKwInOut{Input}{Input}\SetKwInOut{Output}{Output}
	\Input{\textit{$\mathcal{U},\mathcal{I}, \lambda_1,\lambda_2, K$}}
	\Output{Item recommendation matrix  {${A}^*$}}
	\tcc{${A}^*$ is a $n\times N$  binary matrix with elements $1$ when the item is recommended to the user and $0$ otherwise.}
	S $\leftarrow$ Run baseline algorithm and store the relevance score matrix for top-N items for each user \;
	\tcc{$S$ is a $n\times N$  matrix representing the predicted relevance score generated by unfair baseline algorithms for each user-item pair}
	Formulate the Optimization Problem as an Integer Programming following equation \ref{eq:optimization} \;
	Solve Relaxed Linear Programming with added boundary constraint $0\leq A_{ui}\leq1$, $\forall u\in \mathcal{U} , \forall i \in \{1,2,...,N \}$ \;
	\Return{$A^*$}
	\caption{The Fair Re-ranking Optimization}
	\label{alg:optmization}
\end{algorithm}

\begin{algorithm}
	\DontPrintSemicolon
	\SetAlgoLined
	\SetKwInOut{Input}{Input}\SetKwInOut{Output}{Output}
	\Input{\textit{$\mathcal{U},\mathcal{I}, \lambda_1,\lambda_2, K$}}
	\Output{Fair top-K recommendation list  {$L^{F}_K$}}
	
	S $\leftarrow$ Run baseline algorithm and store the relevance score matrix for top-N items for each user  \;
	$UG_{u}$ $\leftarrow$ Binary vector with elements of $1$ when user $u \in \mathcal{U}$ is in the protected group and $-1$ otherwise \;
	$PG_{i}$ $\leftarrow$ Binary vector with elements of $1$ when item ranked $i$ in the traditional top-N recommendation list, ${L}_N(u)$, is in the protected group and $-1$ otherwise \;

	\tcc{Protected groups are the underrepresented groups in baseline algorithms such as unpopular items }
	    
	\ForEach{$u\in \mathcal{U}$}{
        \ForEach{$i\in \{1,2,...,N\}$} {
            $CF_{ui} \leftarrow UG_{u} \times \mathcal{MC}(u,i)$ \;
            $PF_{ui} \leftarrow PG_{i} \times \mathcal{MP}(u,i)$ \;
		    $\hat{S}_{ui}\leftarrow S_{ui} + \lambda_1 CF_{ui} + \lambda_2 PF_{ui} $
		    
		    }
		}
   $L^{F}_K(u)$ $\leftarrow$ Sort matrix $\hat{S}$ along the rows axis and select top-$K$ items for recommendation to each users, $\forall$ $u\in \mathcal{U}$ \;
	\Return{$L^{F}_K(u)$}
	\caption{The Fair Re-ranking Greedy Algorithm}
	\label{alg:LBP}
\end{algorithm}
\updated{On the other hand, the proposed Fair Re-ranking Greedy Algorithm aims to re-rank a list of recommendations in a way that ameliorates biases, particularly for underrepresented or protected groups. This is accomplished by modifying the traditional relevance scores with additional terms that account for user and item group memberships. In the Algorithm~\ref{alg:LBP}, for each user $u \in \mathcal{U}$ and item $i \in \{1,2,...,N\}$, the terms $CF_{ui}$ and $PF_{ui}$ are computed. These terms are indicators of whether a user or item belongs to a protected group. They are multiplied by $\lambda_1$ and $\lambda_2$ respectively, which serve as fairness hyperparameters controlling the extent of the fairness adjustments. If $\lambda_1$ and $\lambda_2$ are chosen judiciously, these adjustments can help rectify biases in the original relevance scores.}

\updated{The algorithm then proceeds to re-rank items based on the adjusted scores, $\hat{S}_{ui}$. The top-$K$ items from this re-ranked list are selected for recommendation. This ensures that the final list of recommendations is not solely based on the original relevance scores, but also takes fairness considerations into account. Theoretically, if the original relevance scores exhibit biases against certain user or item groups, our algorithm is designed to rectify those biases, yielding a fairer ranking. This is because it explicitly incorporates fairness adjustments into the scores and re-ranks items based on these adjusted scores. Further empirical studies are planned to corroborate this theoretical analysis and provide concrete evidence of the algorithm's effectiveness on various datasets.}

We finally discuss the time complexity of the Fair Re-ranking Greedy algorithm.

\begin{lemma}
The time complexity of the proposed re-ranking algorithm has a worst-case bound of $\mathcal{O}(n\times N)$
\end{lemma}

\begin{proof}
The first 3 lines of the algorithm will store the traditional relevance score matrix for Top-N items and store the protected groups in consumers and providers which can be done in linear time $\mathcal{O}(n + m)$. The second part in line 4-10 is to re-score all elements in the original relevance matrix by accounting for fairness measures which needs $\mathcal{O}(n \times N)$ operations in constant time. Lastly, we will sort the new relevance matrix for all users and select top-K items for recommendation. Using max-heap the worst case time complexity is $\mathcal{O}(N + K \log(K))$. Hence, considering $N$ and $K$ are much smaller than $n$, the total time complexity is $\mathcal{O}(n \times N)$ 
\end{proof}
\section{Experimental Methodology}
\label{sec:experimental_methodology}
In this section, we first briefly describe the datasets, baselines and experimental setup used for experiments. All source code and dataset of this project has been released publicly. To foster the reproducibility of our study, we evaluate all recommendations with the Python-based open-source recommendation toolkit Cornac\footnote{\url{https://cornac.preferred.ai/}} \cite{salah2020cornac,truong2021exploring}, and we have made our codes open source.\footnote{\url{https://github.com/rahmanidashti/CPFairRank}}

\partitle{Datasets.}
We use eight public datasets from different domains such as Movie, Point-of-Interest, Music, Book, E-commerce, and Opinion with different types of feedback (\ie explicit or implicit) and characteristics, including \datasetA, \datasetB, \datasetC, and \datasetD. We apply $k$-core pre-processing (\ie each training user/item has at least $k$ ratings) on the datasets to make sure each user/item has sufficient feedback. Table \ref{tbl:datasets} shows the statistics of all datasets. \newtxt{We remind that the size of the datasets employed in our study was deliberately kept modest to ensure \textbf{computational efficiency}. Despite this, the robustness of our reranking model is evident across various datasets, underlining its adaptability to different data volumes and characteristics.}

\begin{table}
  \caption{Statistics of the datasets}
  \centering
  \label{tbl:datasets}
  \begin{tabular}{lcccc}
    \toprule
    \textbf{Dataset} & Users & Items & Interactions & Sparsity \\
    \midrule
    \textbf{\datasetA}  & 943 & 1,349 & 99,287 & 92.19\% \\
    \textbf{\datasetB}  & 2,677 & 2,060 & 103,567 & 98.12\% \\
    \textbf{\datasetC}  & 1,130 & 1,189 & 66,245 & 95.06\% \\
    \textbf{\datasetD}  & 1,797 & 1,507 & 62,376 & 97.69\% \\
    \textbf{BookCrossing}  & 1,136 & 1,019 & 20,522 & 98.22 \% \\
    \textbf{Foursquare}  & 1,568 & 1,461 & 42,678 & 98.13\% \\
    \textbf{AmazonToy}  & 2,170 & 1,733 & 32,852 & 99.12\% \\
    \textbf{AmazonOffice}  & 2,448 & 1,596 & 36,841 & 99.05\% \\
    \bottomrule 
  \end{tabular}
\end{table}

\partitle{Evaluation Method.}
For the experiments, we randomly split each dataset into train, validation, and test sets in 70\%, 10\%, and 20\% respectively, and all the baseline models share these datasets for training and evaluation. We define Consumer-Producer fairness evaluation (mCPF) metric to evaluate the overall performance of models \wrt two-sided fairness. mCPF computes the weighted average deviation of provider fairness (DPF) and deviation of consumer fairness (DCF), \ie $mCPF(w) = w \cdot DPF + (1-w) \cdot DCF$, where $w$ is a weighting parameter that can be selected by the system designer depending on the application or priority for each fairness aspect of the \textit{marketplace}. DCF and DPF are defined in Equations \ref{eq:dpf} and \ref{eq:dcf}, respectively, and the lowest value for mCPF shows the fairest model. To select the model parameters, $\lambda_1$ and $\lambda_2$, we treat them as a hyperparameter and set them to the value that maximizes the average values of mCPF and overall nDCG for each dataset and model.

\partitle{\updated{Base Ranker Methods.}}
We compare the performance of our fairness method on several recommendation approaches, from traditional to deep recommendation models, as suggested by \citet{dacrema2019we}. Therefore, we include two traditional methods (PF and WMF) as well as two deep recommendation models (NeuMF and VAECF). We also include two baselines approaches, Random and MostPop, for further investigation of the results. The introduction of baselines are as the following:  

\begin{itemize}
    \item \textbf{PF} \cite{gopalan2015scalable}: This method is a variant of probabilistic matrix factorization where the weights of each user and item latent features are positive and modeled using the Poisson distribution.
    \item \textbf{WMF} \cite{hu2008collaborative,pan2008one}: This method assigns smaller wights to negative samples and assumes that for two items their latent features are independent.
    \item \textbf{NeuMF} \cite{he2017neural}: This algorithm learns user and item features using multi-layer perceptron (MLP) on one side and matrix factorization (MF) from another side, then applies non-linear activation functions to train the mapping between users and items features that are concatenated from MLP and MF layers.
    \item \textbf{VAECF} \cite{liang2018variational}: This method is based on variational autoencoders which introduce a generative model with multinomial likelihood and use Bayesian inference for parameter estimation.
\end{itemize}

\partitle{Evaluation Settings.}
We adopt the baseline algorithms with the default parameter settings, suggested in their original paper. We set the embedding size for users and items to $50$ for all baseline algorithms. For NeuMF, we set the size of the MLP with $32, 16, 8$ and we apply hyperbolic tangent (TanH) non-linear activation function between layers. We set the learning rate to $0.001$. Early stopping is applied and the best models are selected based on the performance on the validation set. We apply Adam \cite{kingma2014adam} as the optimization algorithm to update the model parameters.

\subsection{Fairness Group Assumption}

\label{sec:fairness_assumption}
The protected groups usually are treated as the advantaged groups in group fairness \cite{pedreschi2009measuring}. The group of users and items can be divided under different requirements for different tasks \cite{abdollahpouri2019unfairness,li2021user}. In this work, we examine two different user grouping methods and one item grouping method. Specifically, we group users based on the level of activity of users and the consumption of popular items in the user profiles. We also group items according to the ratio of their popularity. Hence, we define users and items groups as follows:

\begin{definition}
    \textbf{The level of item popularity (IG):} 
    Following \cite{abdollahpouri2019unfairness,burke2016towards}, we select the top 20\% of items according to the number of interactions they received from the training set as the popular items (\ie short-head items), and the rest as the unpopular items or long-tail items. We apply this grouping strategy to our study to highlight and evaluate the fairness of the recommendation system in terms of how well it can balance the recommendations between popular and unpopular items. In addition, this metric can also be used to evaluate the effectiveness of different recommendation algorithms in terms of their ability to recommend long-tail items.
\end{definition}

\begin{definition}
    \textbf{The level of user activity (UG1):} In order to investigate potential unfairness in recommendation systems, we adopt a user grouping strategy similar to that used by \citet{li2021user}. This strategy involves separating users into two groups based on their level of activity. The top 5\% of users, ranked by the number of interactions, are labelled as the active group, and the remaining users are labelled as the inactive group. By comparing the performance of recommendation algorithms for these two groups, we aim to evaluate the fairness of the system and any potential bias towards active users.
\end{definition}

\begin{definition}
    \textbf{The consumption of popular items in user profile (UG2)}: To further evaluate the fairness of recommendation systems, we use a user grouping strategy similar to that described by \cite{abdollahpouri2019unfairness}. This strategy involves separating users into two groups based on the number of popular items in their profiles. Popular items are defined as those with popularity values that fall within the top 20\% of item popularity. The top 20\% of users, ranked by the number of popular items in their profiles, are labelled as the mainstream group, and the remaining users are labelled as the non-mainstream group. This approach allows us to investigate any potential bias in the recommendations towards users who have a higher consumption of popular items.
\end{definition}

\begin{table*}
\centering
\caption{The recommendation performance of all, and different user and item groups of our re-ranking method and corresponding baselines on \datasetB and \datasetD datasets. All re-ranking results here are obtained under the fairness constraint on nDCG. The evaluation metrics here are calculated based on the top-10 predictions in the test set. Our best results are highlighted in bold. $\Delta$ values denote the percentage of relative improvement in $mCPF$ compared to fairness-unaware recommendation algorithm (N). The reported results for $mCPF$ are based on $w=0.5$. \updated{Note that the implementation of C is aligned with \cite{li2021user} used as baseline here.}}
\label{tbl:resultsseg1}
\begin{adjustbox}{width=\columnwidth}
\begin{tabular}{llllllllllllllll}
\toprule
\multirow{2}{*}{Model} & \multirow{2}{*}{Type} & \multicolumn{4}{c}{User Relevance (nDCG)} && \multicolumn{5}{c}{Item Exposure} && \multicolumn{3}{c}{Both} \\
\cmidrule{3-6}\cmidrule{8-12}\cmidrule{14-16}
                        &                       & All & Active & Inactive & DCF $\downarrow$ && Nov.  & Cov.  & Short. & Long. & DPF $\downarrow$ && mCPF $\downarrow$ & $\frac{mCPF}{All}$ $\downarrow$ & |$\Delta(\%)$| $\uparrow$\\
\midrule
\multicolumn{16}{c}{\textbf{\datasetB}} \\ \hline
PF & N & 0.0321 & 0.0751 & 0.0298 & 0.0453 && 4.9602 & 0.5073 & 0.8281 & 0.1719 & 0.6562 && 0.3588 & 11.1845 & 0.0 \\ 
PF & C & 0.0337 & 0.0732 & 0.0316 & \textbf{0.0415} && \mycc4.9754 & \mycc0.5083 & \mycc0.8243 & \mycc0.1757 & \mycc0.6486 && 0.3524 & 10.4539 & 1.7837 \\ 
PF & P & \mycc 0.0304 & \mycc 0.0757 & \mycc 0.0280 & \mycc 0.0477 && 5.3172 & 0.5325 & 0.6153 & 0.3847 & \textit{0.2306} && 0.1476 & 4.8553 & 58.8629 \\ 
PF & CP & 0.0313 & 0.0740 & 0.0290 & \textit{0.0450} && 5.3302 & 0.5345 & 0.6113 & 0.3887 & \textbf{0.2226} && 0.1418 & 4.5332 & \textbf{60.4794} \\ \hline
WMF & N & 0.0235 & 0.0577 & 0.0217 & 0.0360 && 4.9483 & 0.2913 & 0.9423 & 0.0577 & 0.8846 && 0.4667 & 19.868 & 0.0 \\ 
WMF & C & 0.0236 & 0.0565 & 0.0219 & \textbf{}\textbf{0.0346} && \mycc4.9481 & \mycc0.2898 & \mycc0.9423 & \mycc0.0577 & \mycc0.8846 && 0.4658 & 19.754 & 0.1928 \\ 
WMF & P & \mycc0.0234 & \mycc0.0577 & \mycc0.0215 & \mycc0.0361 && 4.9626 & 0.3019 & 0.9322 & 0.0678 & \textbf{0.8644} && 0.4567 & 19.5505 & 2.1427 \\ 
WMF & CP & 0.0234 & 0.0565 & 0.0217 & \textit{0.0348} && 4.9623 & 0.3010 & 0.9322 & 0.0678 & \textit{0.8644} && 0.4558 & 19.4454 & \textbf{2.3355} \\ \hline
NeuMF & N & 0.0447 & 0.0840 & 0.0427 & 0.0414 && 3.8733 & 0.1223 & 0.9914 & 0.0086 & 0.9828 && 0.5195 & 11.6089 & 0.0 \\ 
NeuMF & C & 0.0456 & 0.0770 & 0.0440 & \textit{0.0330} && \mycc3.8745 & \mycc0.1223 & \mycc0.9913 & \mycc0.0087 & \mycc0.9826 && 0.5137 & 11.2555 & 1.1165 \\ 
NeuMF & P & \mycc0.0448 & \mycc0.0846 & \mycc0.0427 & \mycc0.0420 && 3.8801 & 0.1262 & 0.9849 & 0.0151 & \textbf{0.9698} && 0.5134 & 11.4675 & 1.1742 \\ 
NeuMF & CP & 0.0455 & 0.0749 & 0.0439 & \textbf{0.0310} && 3.8816 & 0.1262 & 0.9851 & 0.0149 & \textit{0.9702} && 0.5061 & 11.1255 & \textbf{2.5794} \\ \hline
VAECF & N & 0.0444 & 0.0801 & 0.0425 & 0.0376 && 4.4061 & 0.3107 & 0.9190 & 0.0810 & 0.8380 && 0.4445 & 10.009 & 0.0 \\ 
VAECF & C & 0.0443 & 0.0735 & 0.04280 & \textbf{0.0308} && \mycc4.4080 & \mycc0.3102 & \mycc0.9184 & \mycc0.0816 & \mycc0.8368 && 0.4393 & 9.9142 & 1.1699 \\ 
VAECF & P & \mycc0.0438 & \mycc0.0789 & \mycc0.0420 & \mycc0.0369 && 4.4349 & 0.3218 & 0.8946 & 0.1054 & \textit{0.7892} && 0.4196 & 9.5777 & 5.6018 \\ 
VAECF & CP & 0.0439 & 0.0739 & 0.0423 & \textit{0.0316} && 4.4366 & 0.3204 & 0.8943 & 0.1057 & \textbf{0.7886} && 0.4157 & 9.4628 & \textbf{6.4792} \\
\midrule
\multicolumn{16}{c}{\textbf{LastFM}} \\ \hline
PF & N & 0.0372 & 0.0259 & 0.0378 & -0.0119 && 5.0905 & 0.6682 & 0.6945 & 0.3055 & 0.3890 && 0.1885 & 5.0699 & 0.0 \\ 
PF & C & 0.0373 & 0.0260 & 0.0379 & \textit{-0.0120} && \mycc5.0904 & \mycc0.6689 & \mycc0.6947 & \mycc0.3053 & \mycc0.3894 && 0.1887 & 5.0590 & 0.1061 \\ 
PF & P & \mycc0.0369 & \mycc0.0250 & \mycc0.0375 & \mycc-0.0125 && 5.1208 & 0.6729 & 0.6792 & 0.3208 & \textit{0.3584} && 0.1729 & 4.6844 & 8.2759 \\ 
PF & CP & 0.0372 & 0.0251 & 0.0378 & \textbf{-0.0127} && 5.1236 & 0.6729 & 0.6784 & 0.3216 & \textbf{0.3568} && 0.1721 & 4.6263 & \textbf{8.7003} \\ \hline
WMF & N & 0.0319 & 0.0498 & 0.0309 & 0.0188 && 5.5360 & 0.7804 & 0.4452 & 0.5548 & -0.1096 && -0.0454 & -1.4236 & 0.0 \\ 
WMF & C & 0.0334 & 0.0425 & 0.0329 & \textbf{0.0096} && \mycc5.5524 & \mycc0.7717 & \mycc0.4406 & \mycc0.5594 & \mycc-0.1188 && -0.0546 & -1.6352 & 20.2643 \\ 
WMF & P & \mycc0.0290 & \mycc0.0454 & \mycc0.0282 & \mycc0.0172 && 5.7963 & 0.7930 & 0.3313 & 0.6687 & \textit{-0.3374} && -0.1601 & -5.5112 & 252.6432 \\ 
WMF & CP & 0.0299 & 0.0422 & 0.0292 & \textit{0.0130} && 5.8066 & 0.7877 & 0.3278 & 0.6722 & \textbf{-0.3444} && -0.1657 & -5.5455 & \textbf{264.978} \\ \hline
NeuMF & N & 0.0415 & 0.0394 & 0.0416 & -0.0022 && 3.7975 & 0.0916 & 0.9940 & 0.0060 & 0.9880 && 0.4929 & 11.8828 & 0.0 \\ 
NeuMF & C & 0.0444 & 0.0348 & 0.0449 & \textbf{-0.0101} && \mycc3.8015 & \mycc0.0929 & \mycc0.9940 & \mycc0.0060 & \mycc0.9880 && 0.4889 & 11.0113 & 0.8115 \\ 
NeuMF & P & \mycc0.0411 & \mycc0.0391 & \mycc0.0412 & \mycc-0.0021 && 3.8122 & 0.0936 & 0.9811 & 0.0189 & \textit{0.9622} && 0.4801 & 11.6955 & 2.5969 \\
NeuMF & CP & 0.0436 & 0.0303 & 0.0443 & \textit{-0.0140} && 3.8192 & 0.0942 & 0.9773 & 0.0227 & \textbf{0.9546} && 0.4703 & 10.7966 & \textbf{4.5851} \\ \hline
VAECF & N & 0.0560 & 0.0651 & 0.0556 & 0.0095 && 4.6040 & 0.4293 & 0.7730 & 0.2270 & 0.5460 && 0.2777 & 4.9572 & 0.0 \\ 
VAECF & C & 0.0598 & 0.0550 & 0.0600 & \textit{-0.0050} && \mycc4.6036 & \mycc0.4313 & \mycc0.7738 & \mycc0.2262 & \mycc0.5476 && 0.2713 & 4.5375 & 2.3046 \\ 
VAECF & P & \mycc0.0544 & \mycc0.0589 & \mycc0.0542 & \mycc0.0047 && 4.6660 & 0.4340 & 0.7294 & 0.2706 & \textbf{0.4588} && 0.2317 & 4.2561 & 16.5646 \\
VAECF & CP & 0.0561 & 0.0549 & \textit{0.0562} & \textbf{-0.0013} && 4.6635 & 0.4360 & 0.7318 & 0.2682 & \textit{0.4636} && 0.2311 & 4.1172 & \textbf{16.7807} \\ \hline

\bottomrule
\end{tabular}
\end{adjustbox}
\end{table*}

\begin{table*}
\centering
\caption{The recommendation performance of all, and different user and item groups of our re-ranking method and corresponding baselines on \newtxt{\datasetB and \datasetD} datasets. All re-ranking results here are obtained under the fairness constraint on nDCG. The evaluation metrics here are calculated based on the top-10 predictions in the test set. Our best results are highlighted in bold. $\Delta$ values denote the percentage of relative improvement in $mCPF$ compared to fairness-unaware recommendation algorithm (N). The reported results for $mCPF$ are based on $w=0.5$. \updated{Note that the implementation of C is aligned with \cite{li2021user} used as baseline here.}}
\label{tbl:resultsseg2}
\begin{adjustbox}{width=\columnwidth}
\begin{tabular}{llllllllllllllll}
\toprule
\multirow{2}{*}{Model} & \multirow{2}{*}{Type} & \multicolumn{4}{c}{User Relevance (nDCG)} && \multicolumn{5}{c}{Item Exposure} && \multicolumn{3}{c}{Both} \\
\cmidrule{3-6}\cmidrule{8-12}\cmidrule{14-16}
                        &                       & All & \newtxt{Main.} & \newtxt{Non-main.} & DCF $\downarrow$ && Nov. & Cov. & Short. & Long. & DPF $\downarrow$ && mCPF $\downarrow$ & $\frac{mCPF}{All}$ $\downarrow$ & |$\Delta(\%)$| $\uparrow$\\
\midrule
\multicolumn{16}{c}{\textbf{\datasetB}} \\ \hline
PF & N & 0.0321 & 0.0369 & 0.0309 & 0.0060 && 4.9602 & 0.5073 & 0.8281 & 0.1719 & 0.6562 && 0.3399 & 10.5954 & 0.0 \\ 
PF & C & 0.0331 & 0.0338 & 0.0330 & \textbf{0.0008} && \mycc4.9633 & \mycc0.5058 & \mycc0.8265 & \mycc0.1735 & \mycc0.6530 && 0.3332 & 10.0513 & 1.9712 \\ 
PF & P & \mycc0.0304 & \mycc0.0349 & \mycc0.0293 & \mycc0.0056 && 5.3172 & 0.5325 & 0.6153 & 0.3847 & \textit{0.2306} && 0.1229 & 4.0428 & 63.8423 \\ 
PF & CP & 0.0309 & 0.0323 & 0.0306 & \textit{0.0017} && 5.3224 & 0.5325 & 0.6133 & 0.3867 & \textbf{0.2266} && 0.117 & 3.7876 & \textbf{65.5781} \\ \hline
WMF & N & 0.0235 & 0.0167 & 0.0252 & -0.0085 && 4.9484 & 0.2913 & 0.9423 & 0.0577 & 0.8846 && 0.4420 & 18.8165 & 0.0 \\ 
WMF & C & 0.0236 & 0.0163 & 0.0255 & \textbf{-0.0092} && \mycc4.9479 & \mycc0.2908 & \mycc0.9423 & \mycc0.0577 & \mycc0.8846 && 0.4413 & 18.6912 & 0.1584 \\
WMF & P & \mycc0.0234 & \mycc0.0166 & \mycc0.0250 & \mycc-0.0085 && 4.9627 & 0.3019 & 0.9322 & 0.0678 & \textbf{0.8644} && 0.4317 & 18.4803 & 2.3303 \\ 
WMF & CP & 0.0234 & 0.0162 & 0.0252 & \textit{-0.0090} && 4.9619 & 0.3005 & 0.9323 & 0.0677 &\textit{ 0.8646} && 0.4313 & 18.4316 & \textbf{2.4208} \\ \hline
NeuMF & N & 0.0451 & 0.047 & 0.0446 & 0.0025 && 3.8923 & 0.1214 & 0.9911 & 0.0089 & 0.9822 && 0.5025 & 11.1518 & 0.0 \\ 
NeuMF & C & 0.0462 & 0.0462 & 0.0462 & \textbf{0.0000} && \mycc3.8945 & \mycc0.1233 & \mycc0.9911 & \mycc0.0089 & \mycc0.9822 && 0.5000 & 10.8108 & 0.4975 \\ 
NeuMF & P & \mycc0.0451 & \mycc0.047 & \mycc0.0446 & \mycc0.0024 && 3.9049 & 0.1277 & 0.9786 & 0.0214 & \textbf{0.9572} && 0.4896 & 10.8655 & 2.5672 \\ 
NeuMF & CP & 0.0462 & 0.0464 & 0.0462 & \textit{0.0002} && 3.9075 & 0.1272 & 0.9786 & 0.0214 & \textit{0.9572} && 0.4875 & 10.5428 & \textbf{2.9851} \\ \hline
VAECF & N & 0.0444 & 0.0527 & 0.0423 & 0.0103 && 4.3809 & 0.2893 & 0.9229 & 0.0771 & 0.8458 && 0.4407 & 9.9279 & 0.0 \\ 
VAECF & C & 0.0464 & 0.0175 & 0.0536 & \textbf{-0.0362} && \mycc4.4419 & \mycc0.3019 & \mycc0.9153 & \mycc0.0847 & \mycc0.8306 && 0.3873 & 8.3488 & 12.1171 \\ 
VAECF & P & \mycc0.0376 & \mycc0.0461 & \mycc0.0355 & \mycc0.0106 && 4.9044 & 0.318 & 0.5856 & 0.4144 & \textbf{0.1712} && 0.0976 & 2.5944 & 77.8534 \\ 
VAECF & CP & 0.0388 & 0.0186 & 0.0439 & \textit{-0.0253} && 4.9316 & 0.3218 & 0.5856 & 0.4144 & \textit{0.1712} && 0.0623 & 1.6048 & \textbf{85.8634} \\
\midrule
\multicolumn{16}{c}{\textbf{LastFM}} \\ \hline
PF & N & 0.0372 & 0.0507 & 0.0338 & 0.0169 && 5.0905 & 0.6682 & 0.6945 & 0.3055 & 0.3890 && 0.2054 & 5.5245 & 0.0 \\ 
PF & C & 0.0372 & 0.0500 & 0.0340 & \textbf{0.0160} && \mycc5.0883 & \mycc0.6682 & \mycc0.6951 & \mycc0.3049 & \mycc0.3902 && 0.2055 & 5.5227 & 0.0487 \\ 
PF & P & \mycc0.0369 & \mycc0.0509 & \mycc0.0334 & \mycc0.0175 && 5.1208 & 0.6729 & 0.6792 & 0.3208 & 0.3584 && 0.1902 & 5.1531 & 7.4002 \\ 
PF & CP & 0.0370 & 0.0497 & 0.0339 & \textit{0.0159} && 5.1222 & 0.6729 & 0.6784 & 0.3216 & \textbf{0.3568} && 0.1886 & 5.0932 & \textbf{8.1792} \\ \hline
WMF & N & 0.0319 & 0.0522 & 0.0268 & 0.0254 && 5.5358 & 0.7804 & 0.4452 & 0.5548 & \textit{-0.1096} && -0.0427 & -1.339 & 0.0 \\ 
WMF & C & 0.0328 & 0.0502 & 0.0284 & \textbf{0.0218} && \mycc5.5379 & \mycc0.7711 & \mycc0.4443 & \mycc0.5557 & \mycc-0.1114 && -0.0454 & -1.3854 & 6.3232 \\ 
WMF & P & \mycc0.0290 & \mycc0.0443 & \mycc0.0252 & \mycc0.0190 && 5.7961 & 0.7930 & 0.3313 & 0.6687 & \textbf{-0.3374} && -0.1612 & -5.5491 & 277.5176 \\ 
WMF & CP & 0.0295 & 0.0435 & 0.0260 & \textit{0.0175} && 5.7963 & 0.7863 & 0.3306 & 0.6694 & \textit{-0.3388} && -0.1626 & -5.5137 & \textbf{280.7963} \\ \hline
NeuMF & N & 0.0415 & 0.0720 & 0.0338 & 0.0382 && 3.7975 & 0.0916 & 0.9940 & 0.0060 & 0.9880 && 0.5192 & 12.5169 & 0.0 \\ 
NeuMF & C & 0.0408 & 0.0622 & 0.0355 & \textbf{0.0267} && \mycc3.7906 & \mycc0.0942 & \mycc0.9938 & \mycc0.0062 & \mycc0.9876 && 0.5132 & 12.5784 & 1.1556 \\ 
NeuMF & P & \mycc0.0411 & \mycc0.0698 & \mycc0.0339 & \mycc0.0359 && 3.8122 & 0.0936 & 0.9811 & 0.0189 & \textbf{0.9622} && 0.5050 & 12.3021 & 2.735 \\ 
NeuMF & CP & 0.0400 & 0.058 & 0.0355 & \textit{0.0225} && 3.8052 & 0.0949 & 0.9809 & 0.0191 & \textit{0.9618} && 0.4981 & 12.4556 & \textbf{4.0639} \\ \hline
VAECF & N & 0.0555 & 0.0779 & 0.0499 & 0.0280 && 4.6164 & 0.4280 & 0.7790 & 0.2210 & 0.5580 && 0.2965 & 5.3414 & 0.0 \\ 
VAECF & C & 0.0538 & 0.0002 & 0.0672 & \textbf{-0.067} && \mycc4.5467 & \mycc0.4333 & \mycc0.7821 & \mycc0.2179 & \mycc0.5642 && 0.2518 & 4.6794 & 15.0759 \\ 
VAECF & P & \mycc0.0509 & \mycc0.0748 & \mycc0.0449 & \mycc0.0298 && 5.1853 & 0.4605 & 0.3776 & 0.6224 & \textbf{-0.2448} && -0.1089 & -2.1391 & 136.7285 \\ 
VAECF & CP & 0.0470 & 0.0106 & 0.0561 & \textit{-0.0454} && 5.1287 & 0.4612 & 0.3776 & 0.6224 & \textit{-0.2448} && -0.1467 & -3.1233 & \textbf{149.4772} \\ 
\bottomrule
\end{tabular}
\end{adjustbox}
\end{table*}

The reason for exploring these two methods is that we believe the difference between user interactions and popular consumption will reflect their different activity from both quantitative and qualitative positives and thus in a more reasonable manner.
\section{Results and Discussion}
\label{sec:results}

\subsection{Statements of examined research questions}
To gain a better understanding of the merits of the proposed CP fairness optimization framework, through the course of experiments, we intend to answer the following evaluation questions:

\begin{description}
    \item \textbf{RQ1}: In this research question, we wish to investigate whether there is a trade-off between the three evaluation objectives <Acc, DCF, DPF> in the \underline{baseline} CF models (i.e., the rows denoted with \squotes{N} indicating potentially unfair models), where trade-off may be a result of underlying data biases and may require \textit{algorithmic intervention}. \\
    
    \item \textbf{RQ2}: In the second research question, we seek to determine whether the popular \underline{uni-sided} fairness focused on C-Fairness or P-Fairness intervention is sufficient to reduce unfairness on both the consumer- and producer-sides, or whether biases persist in one side.  This will motivate the potential necessity of \textit{two-sided fairness intervention}.  \\

    \item \textbf{RQ3}. This research question investigates whether or not the proposed \underline{CP-fairness} algorithm based on MIP in this work satisfies the requirements for fairness on \underline{both} the consumer and supplier sides concurrently, and whether or not this compromises the \underline{overall accuracy} of the system.

 \end{description}

We will now proceed to the next section to describe the answers to these research questions in further detail.

\begin{table}[h]
\caption{The correlation coefficients between evaluation metrics <Acc, DCF> and <Acc, DPF> across all
eight datasets were explored in this research (for baseline models in the N category). \label{tbl:correlation}}
\begin{tabular}{c|c|c|c}
\toprule
 & model/obj. & \multicolumn{1}{c|}{<Acc, DCF>} & \multicolumn{1}{c}{<Acc, DPF>} \\ \midrule
\multirow{2}{*}{\begin{tabular}[c]{@{}c@{}}Activity-lev.\\ (Table 3)\end{tabular}} & VACEF & 0.0366 & -0.0511 \\
 & NeuMF & 0.7494 & 0.0356 \\ \midrule
\multirow{2}{*}{\begin{tabular}[c]{@{}c@{}}Mainstream\\ (Table 4)\end{tabular}} & VACEF & -0.1664 & -0.3679 \\ 
 & NeuMF & 0.4497 & -0.0836 \\
 \bottomrule
\end{tabular}
\end{table}

\subsection{Answers to the RQs and Discussion}

The result of applying our proposed \texttt{CP-FairRank} optimization algorithm in detail are presented in Table~\ref{tbl:resultsseg1} (for activity-level user segmentation) and Table~\ref{tbl:resultsseg2} (for mainstream-based user segmentation -- main.~vs.~non-main.), where both tables employ the same popularity-based segmentation on the item side. In both experimental settings, neural models \textbf{NeuMF} and \textbf{VAEFCF} tend to achieve the highest levels of accuracy. As a result, we focus on these models when describing fairness performance, as it would be pointless to analyze fairness for models with relatively poor performance. \\

\vspace{-3mm}

\noindent \textbf{Answers to RQ1 and RQ2.} \dquotes{\textit{Bias inherent in the data of unfair models and the effect of applying one-sided fairness optimization on evaluation objectives \textbf{<Acc, DCF, DPF>}}}.

\begin{itemize}
    \item First, we show some individual examples from Table~\ref{tbl:resultsseg1} and Table~\ref{tbl:resultsseg2}. According to Table~\ref{tbl:resultsseg2}, one can note the results of the best models along $<Acc, DCF, DPF>$ for the \textbf{Epinion} dataset: NeuMF: $<0.0451, 0.0025, 0.9822>$ and VAECF: $<0.0444, 0.0103, 0.8458>$, while for \textbf{LastFM} dataset, we have NeuMF: $<0.0415, 0.00382, 0.988>$ and VAECF: $<0.0555, 0.028, 0.558>$. One immediate observation is that in at least these two datasets, these models alter DCF and DPF inconsistently to obtain high performance on accuracy. As an illustration, in \textbf{Epinion}, VAECF achieves good performance with less compromise on DPF but a higher compromise on DPF, and this trend is \textit{reversed} in \textbf{LastFM}. These findings imply that preexisting biases in the data may be exaggerated differently by the baseline CF model, necessitating algorithmic intervention.

    \item Since Table~\ref{tbl:resultsseg1}, and Table~\ref{tbl:resultsseg2}, only represent a small subset of the total experiments conducted, in Table \ref{tbl:correlation} we compute the Pearson correlation coefficient across all eight datasets examined in this study better to comprehend the relationship between the various evaluation objectives. Ideally, a negative correlation between $<Acc\uparrow, DPF\downarrow>$ and $<Acc\uparrow, DCF\downarrow>$ is desirable, as we need maximal accuracy and minimal unfairness on the producer and consumer side. The most intriguing finding is for the VACEF method, which demonstrates a negative correlation between $<Acc, DCF>$, where this value is subtly better (i.e., lower) than its NeuMF equivalent, Table \ref{tbl:resultsseg1} (-0.0511 vs.~-0.356) vs. Table \ref{tbl:resultsseg2} (0.356 vs.~-0.0836). Quite intriguingly, we observe a strong positive correlation between and, as opposed to a weak or negative correlation in VACEF. These results indicate that \textit{VACEF is a strategy that employs distinct mechanisms for retrieving relevant items}, is capable of recommending relevant by proposing items from the long tail (thereby enhancing DPF), and better serves the under-served user group (improving DCF).
    \item Now, we will focus on RQ2 and investigate the impact of one-sided fairness optimization on several evaluation criteria. Table~\ref{tbl:resultsseg1}  and Table~\ref{tbl:resultsseg2}  reveal that the general trend is that C-optimization increases C-Fairness while P-optimization improves P-Fairness. For instance, from Table \ref{tbl:resultsseg1}, we can see that NeuMF C-optimization could achieve almost 20\% improvement compared to the baseline (\ie N) on Epinion. Moreover, WMP P-Optimisation (see Table \ref{tbl:resultsseg2}) improved the baseline models by 207\%. To account for variation among datasets, we calculated the coefficients $\frac{DCF}{ACC}$ and $\frac{DPF}{ACC}$ for the two best models in Table \ref{tbl:tradeoff}. We can observe that the general trend is that C-Fairness optimization improves DCF, while P-Fairness optimization improves DPF and this improvement appears to be greater for VACEF. Another intriguing observation is that uni-sided fairness optimization may not guarantee multi-sided fairness improvement, e.g., P-optimization may have a negative impact on  $\frac{DCF}{ACC}$, NeuMF (1.091 to 1.094) and $\frac{DCF}{ACC}$, VAECF (0.255 to 0.290).
\end{itemize}

\noindent These results are quite insightful and demonstrate that our suggested strategy is capable of addressing diverse stakeholder fairness concerns. They also expose a gap in past research, namely, the failure to investigate \textit{multi-sided fairness evaluation} when considering single-sided fairness intervention. Thus, the results in this part set the ground for simultaneously demonstrating the performance of our proposed \texttt{CP-FairRank} on all stakeholder evaluation objectives. \\

\begin{table}[h]
\caption{We calculated the $\frac{DPF}{ACC}$ and $\frac{DCF}{ACC}$ for each of the 8 datasets utilized in this study, and calculated their average below, in order to demonstrate the interplay and probable trade-off between fairness and overall system accuracy. \label{tbl:tradeoff}}
\begin{tabular}{c|c|ccc|ccc}
\toprule
 & \multirow{2}{*}{model/obj.} & \multicolumn{3}{c|}{$\frac{DPF}{ACC}\downarrow$} & \multicolumn{3}{c}{$\frac{DCF}{ACC}\downarrow$} \\
\cline{3-8}
 &  & \multicolumn{1}{c}{N} & \multicolumn{1}{c}{P} & C & \multicolumn{1}{c}{N} & \multicolumn{1}{c}{C} & P \\ \hline
\multirow{2}{*}{\begin{tabular}[c]{@{}c@{}}Activity-lev.\\ (Table 3)\end{tabular}} & VACEF & \multicolumn{1}{c}{21.238} & \multicolumn{1}{c}{18.786} & 21.006 & \multicolumn{1}{c}{1.051} & \multicolumn{1}{c}{0.942} & 1.002 \\
 & NeuMF & \multicolumn{1}{c}{50.130} & \multicolumn{1}{c}{49.891} & 49.805 & \multicolumn{1}{c}{1.091} & \multicolumn{1}{c}{0.782} & 1.094 \\ \midrule
\multirow{2}{*}{\begin{tabular}[c]{@{}c@{}}Mainstream\\ (Table 4)\end{tabular}} & VACEF & \multicolumn{1}{c}{20.869} & \multicolumn{1}{c}{0.976} & 21.499 & \multicolumn{1}{c}{0.255} & \multicolumn{1}{c}{-0.802} & 0.290 \\
 & NeuMF & \multicolumn{1}{c}{46.498} & \multicolumn{1}{c}{46.232} & 46.201 & \multicolumn{1}{c}{0.314} & \multicolumn{1}{c}{0.182} & 0.308 \\
 \bottomrule
\end{tabular}
\end{table}

\noindent \textbf{Answers to RQ3.} \dquotes{\textit{Examining the effectiveness of \texttt{CP-FairRank}, our proposed CP-Fairness optimization method for multi-stakeholder fairness optimization}}.

The response to this research question is the center of this study. Here are some insights into the effectiveness of \texttt{CP-FairRank}, our CPFairness optimization method:

\begin{itemize}
    \item Upon first inspection of Tables \ref{tbl:resultsseg1} and \ref{tbl:resultsseg2}, it is evident that the suggested CP-Fairness optimization approach is able to simultaneously improve $DCF$ and $DPF$ of the investigated CF models without sacrificing a great deal of accuracy, and this holds for both datasets and disadvantaged groups. Due to their simplicity, we employ the statistics $mCPF$ and $\frac{mCPF}{Acc}$ to show this effect (the lower these value are the better they are judged). For instance, we compare here these two statistics for \squotes{N} and \squotes{CP} for a randomly selected dataset \textbf{Epinion} and VACEF method, which is equal to VAECF-N = <0.4445, 10.0009> and VAECF-CP = <0.4157, 9.4628>, whereas for the second user group VAECF-N = <0.4407, 9.9279> and VAECF-CP = <0.0623, 1.6048>. These results demonstrate a clear improvement in joint fairness without affecting system utility as a whole, and this improvement is significantly greater for the second user grouping.
     \item We propose to aggregate the results across all datasets using the plots in Fig.~\ref{fig:1stUG}, where the x-axis represents the overall system utility (accuracy) and the y-axis shows one of fairness objectives (DPF or DCF). In this case, the plots are the baseline CF model without any fairness intervention (\squotes{N}) and the models after making \squotes{CP} intervention using our proposed \texttt{CP-FairRank}. Note that the points in these graphs reflect one of the 8 datasets examined in this study.\footnote{Consider that the orange plots (after fairness intervention) should be below the blue plots (before intervention) since the y-axis represents the fairness targets and the lower they are, the lower the unfairness.} This tendency can be easily confirmed by examining the results across all models, demonstrating that our strategy can increase fairness consistently across models and datasets. In certain circumstances, for instance, NeuMF, the improvement appears to be greater (looking at the sloping curve downward), indicating that, depending on how much each model produces biased outcome, there may be more or less space for improvement.
\end{itemize}

\vspace{3mm}

\colorbox{gray!30}{
\begin{minipage}{0.94\textwidth}
\noindent \textbf{Summary.}
To provide a fair recommendation to both disadvantaged consumer and producer groups, we propose \texttt{CP-FairRank}, a post-hoc method for optimizing the fairness of ranking lists provided to users by taking into account the interests of both stakeholders involved in recommendation, consumers and producers. Through large-scale experiments on 8 datasets broadly adopted used in recommendation research, we believe our study provides persuasive evidence that the proposed technique can increase mCPF (joint consumer and producer fairness) in literally all experimental case without compromising much on accuracy (in some case the accuracy is also improved). Notable improvement is also shown in the user grouping based on \textit{main-streamness}, indicating that fairness optimization is not a purely mathematical process and that its quality can be affected by the data distribution provided by the underlying user and item groups.
\end{minipage}}

\vspace{3mm}

\begin{figure*}
  \centering
  \subfloat[NeuMF <Acc, DCF>]
    {\includegraphics[scale=0.25]{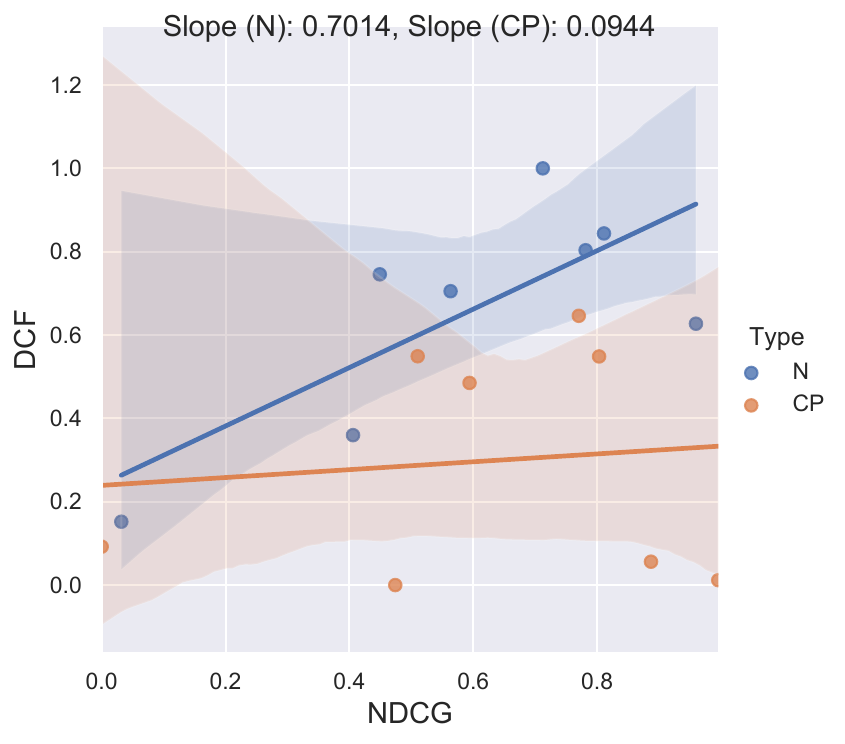}
    }
  \hfill
  \subfloat[NeuMF <Acc, DPF>]
    {\includegraphics[scale=0.25]{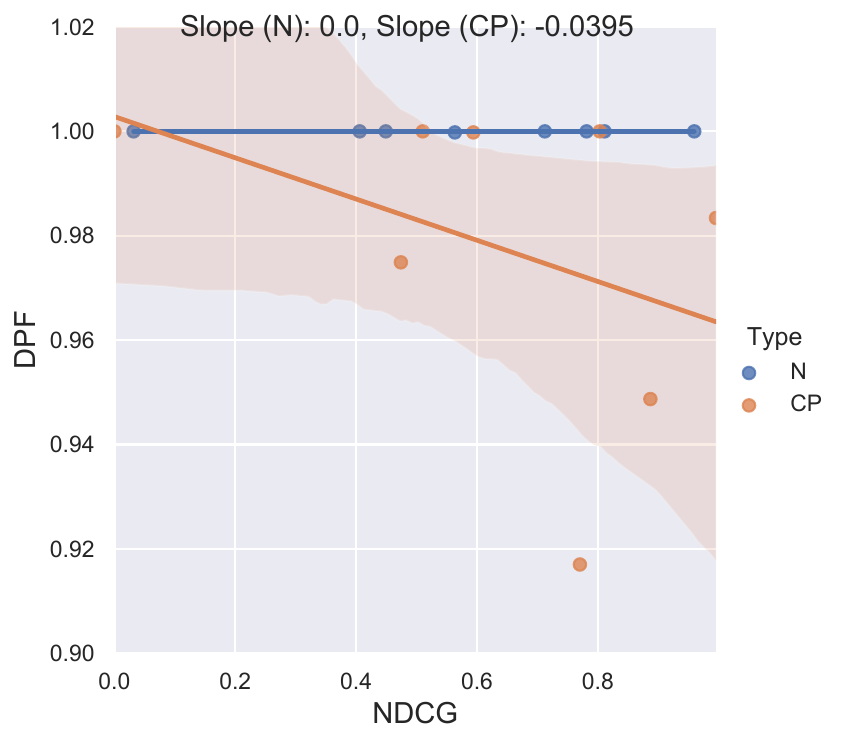}
    }
    \hfill
  \subfloat[VAECF <Acc, DCF>]
    {\includegraphics[scale=0.25]{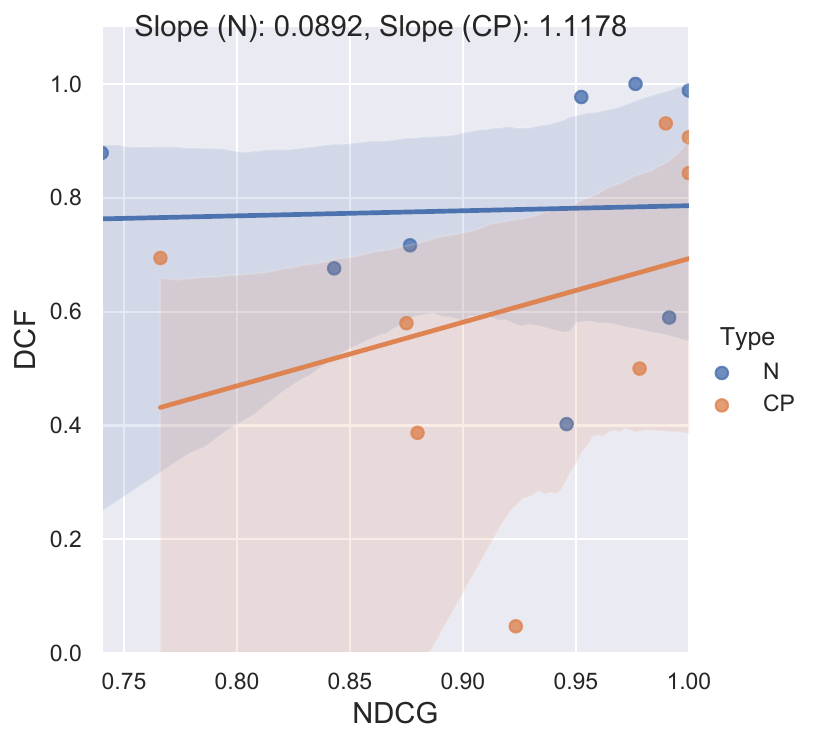}
    }
   \hfill
  \subfloat[VAECF <Acc, DPF>]
    {\includegraphics[scale=0.25]{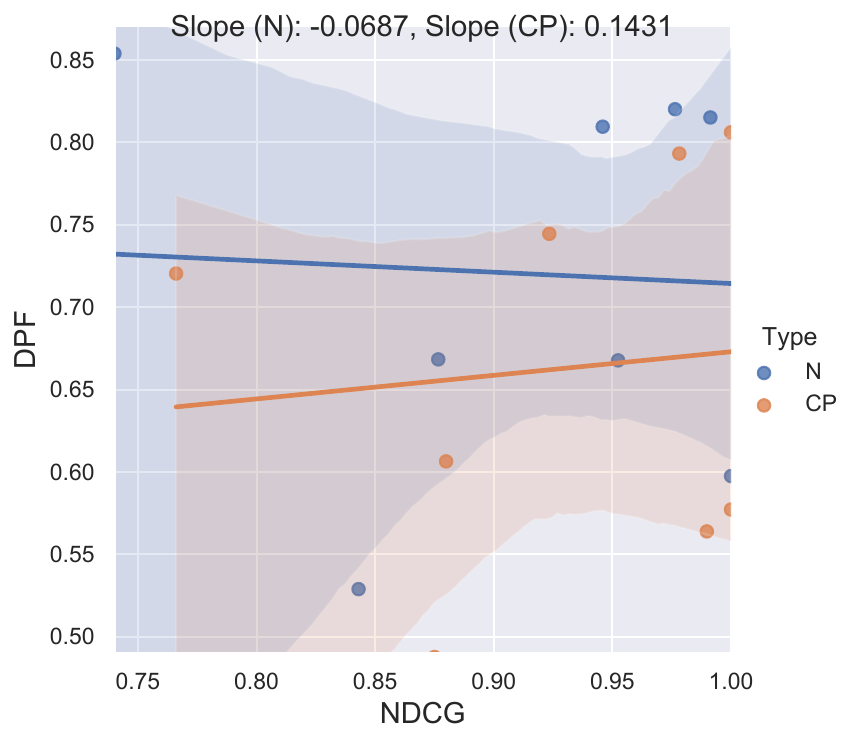}
    }
      \caption{Correlation plots were produced for the \textit{best performing} models (in terms of accuracy), NeuMF, and VACEF, to illustrate the potential trade-off between accuracy and fairness metrics. Noting that the points in this example represent datasets, the correlation measures the generalizability of the findings across domains.}
      \label{fig:1stUG}
\end{figure*}

\partitle{Ablation Study: Distributions of mCPF.}
In order to thoroughly assess the effectiveness of our proposed CP-fairness re-ranking algorithm, we conducted a comprehensive experiment with 128 cases on eight real-world datasets. The experiment was conducted on two user groups, UG1 and UG2, to examine the ability of the algorithm to meet the requirements of both user and item fairness in a two-sided marketplace. The results of the experiment are presented in Figures \ref{fig:CPevalBoxPlotUG1} and \ref{fig:CPevalBoxPlotUG2}. These figures show the distribution of the mCPF performance for all datasets, using four baseline algorithms (PF, WMF, NeuMF, and VAECF) on activity-level and mainstream user groups, respectively. The results indicate that our proposed CP-fairness model outperforms the baseline algorithms and unilateral fairness models in all datasets, with a lower mCPF. It is also important to note that the behavior of the unilateral fairness models varies across datasets. This variability is due to differences in the underlying data characteristics. For example, in Fig.~\ref{fig:CPevalBoxPlot1_MovieLens}, the C-fairness method performs better than the P-fairness method, while in Fig.~\ref{fig:CPevalBoxPlot1_Fourquare} on the Foursquare dataset, the P-fairness approach outperforms the C-fairness approach. In conclusion, the results of our experiment demonstrate that our CP re-ranking algorithm not only reduces fairness disparities between the two groups of users and items, but also provides competitive accuracy in comparison to the unilateral fairness methods and baselines. This is reflected in the nDCG@10 scores reported on top of the plots in Figures \ref{fig:CPevalBoxPlotUG1} and \ref{fig:CPevalBoxPlotUG2}. Our proposed algorithm is therefore a promising solution for enhancing fairness in two-sided marketplaces.

\partitle{Ablation Study: Effect of $\lambda_1$ and $\lambda_2$.}
\label{sec:ablation}
We also analyze the impact of the optimization regularization parameters $\lambda_1$ and $\lambda_2$ from Equation \ref{eq:optimization}. We expect that the bigger $\lambda_1$ and $\lambda_2$ are, the fairer our proposed model will be. However, the excessive pursuit of fairness is unnecessary and could have an adverse impact on the overall recommendation performance. Therefore, we are interested in studying how different values of $\lambda_1$ and $\lambda_2$ in Equation \ref{eq:optimization} can influence the total performance of the system, \ie total accuracy, and novelty of the recommendation, as well as their impact on consumer and provider groups.

\begin{figure*}
  \centering
  \subfloat[MovieLens]
    {\includegraphics[scale=0.24]{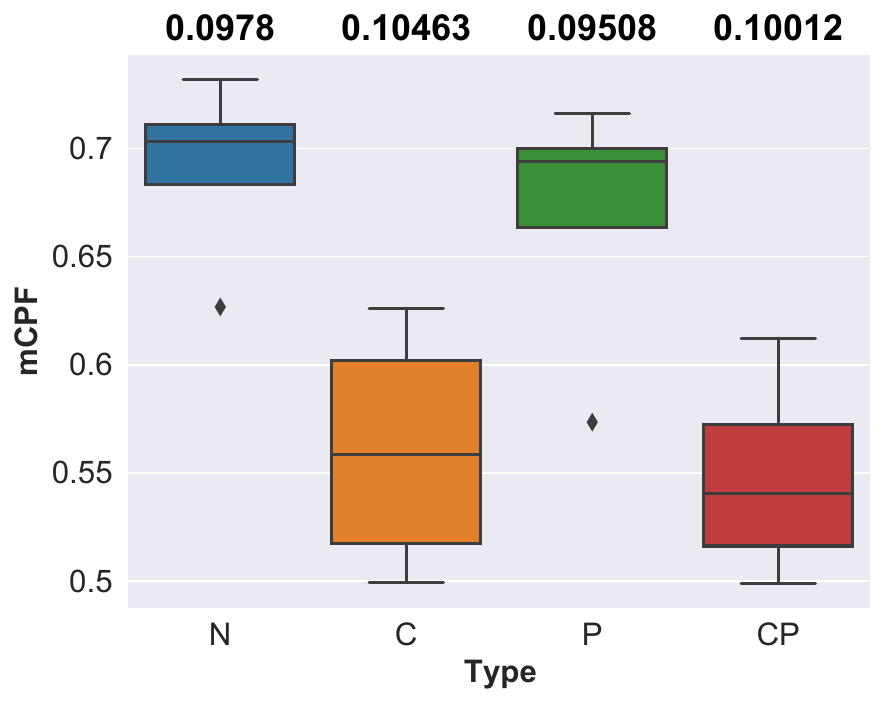}
    \label{fig:CPevalBoxPlot1_MovieLens}}
  \hfill
  \subfloat[Epinion]
    {\includegraphics[scale=0.24]{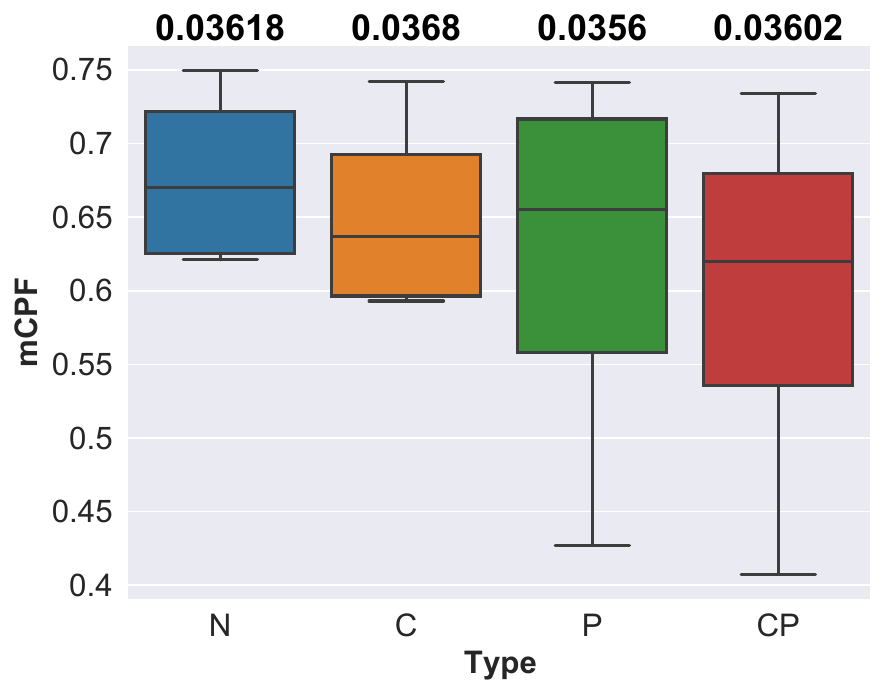}
    \label{fig:CPevalBoxPlot1_Epinion}}
  \hfill
  \subfloat[AmazonToy]
    {\includegraphics[scale=0.24]{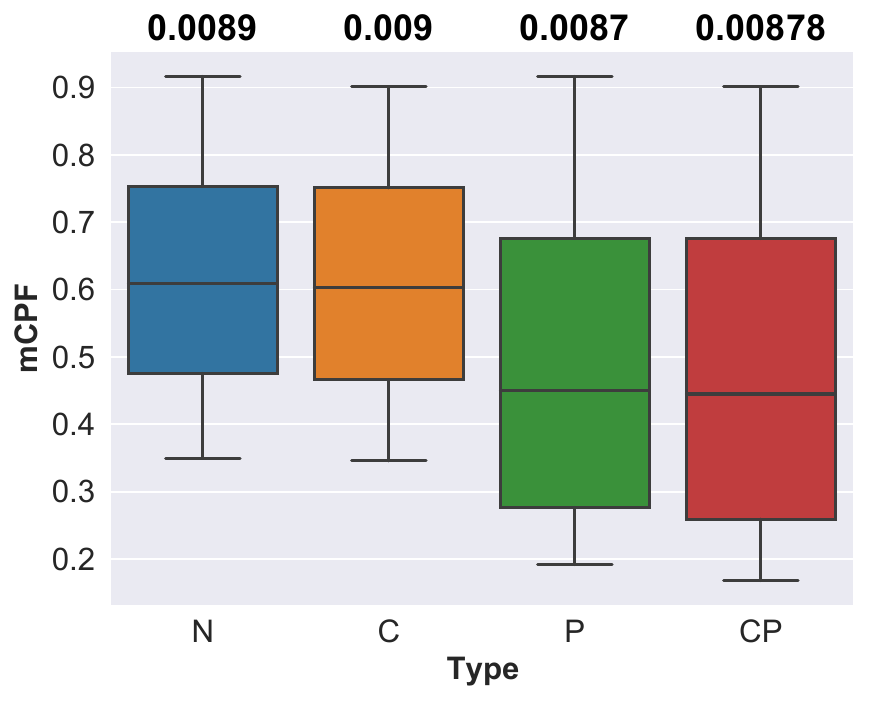}
    \label{fig:CPevalBoxPlot1_AmazonToy}}
    \hfill
  \subfloat[AmazonOffice]
    {\includegraphics[scale=0.24]{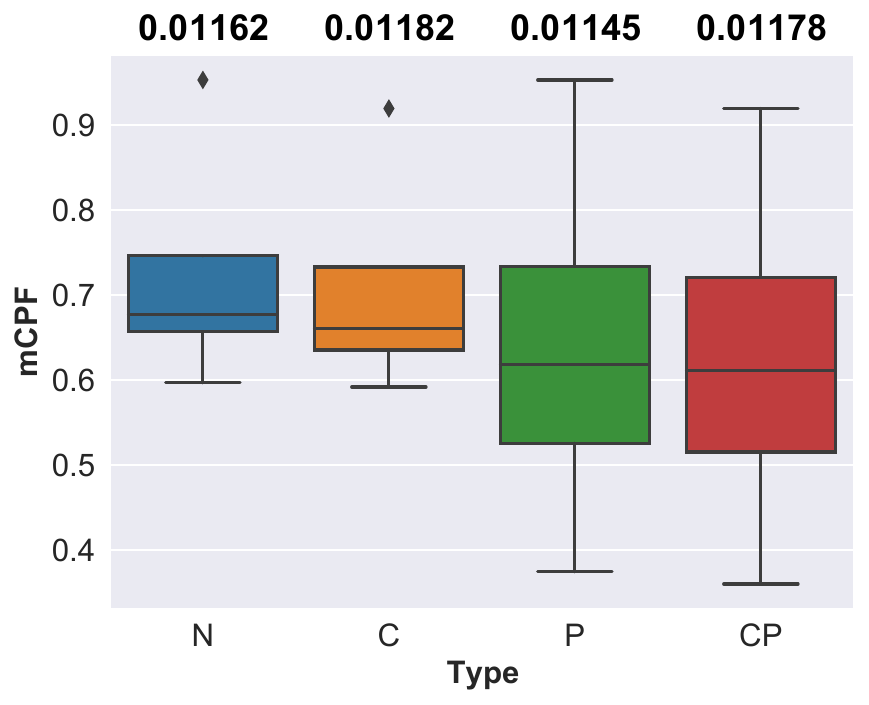}
    \label{fig:CPevalBoxPlot1_AmazonOffice}}
   \hfill
  \subfloat[BookCrossing]
    {\includegraphics[scale=0.24]{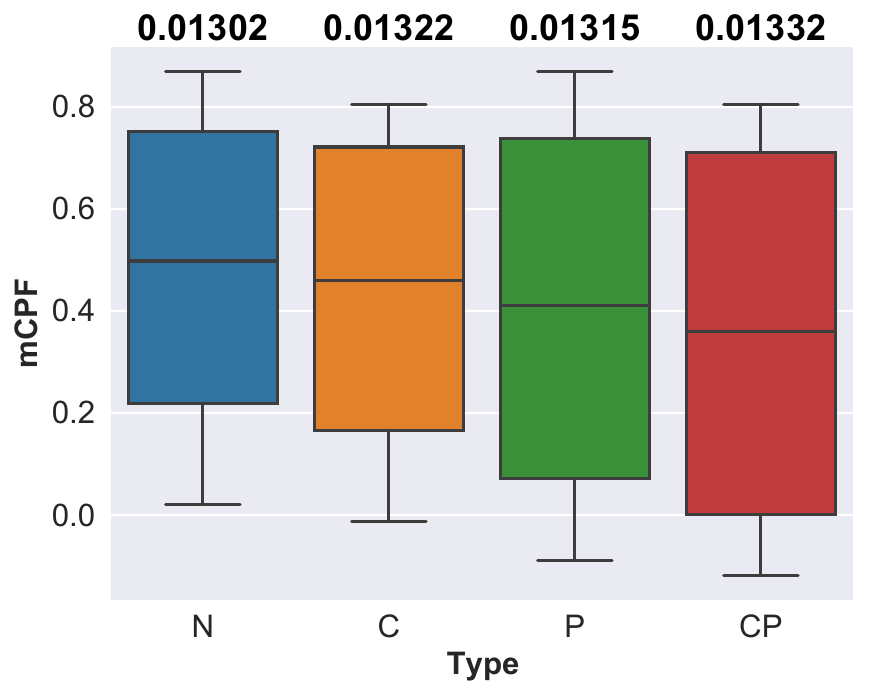}
    \label{fig:CPevalBoxPlot1_BookCrossing}}
  \hfill
  \subfloat[Gowalla]
    {\includegraphics[scale=0.24]{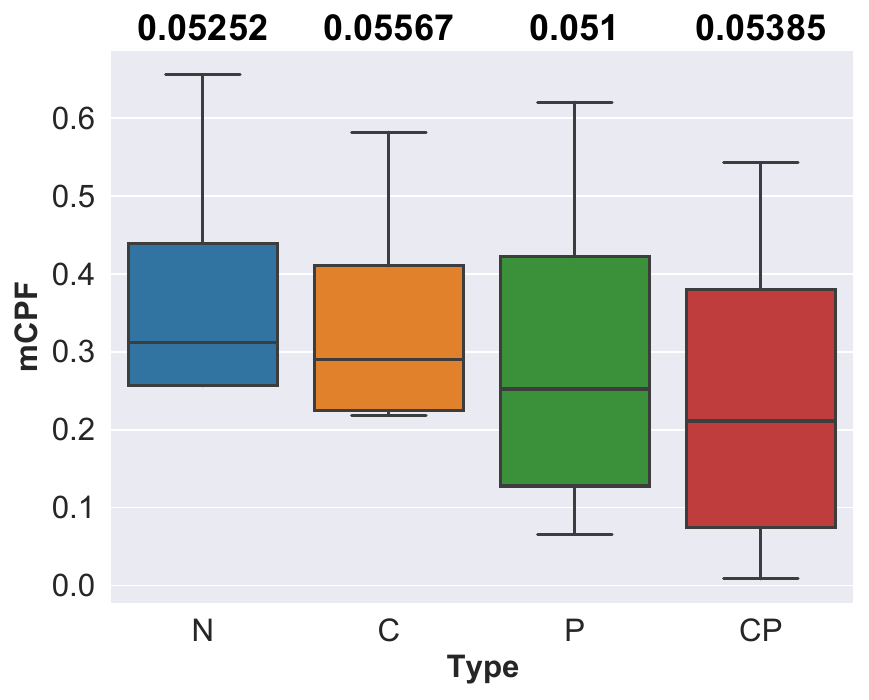}
    \label{fig:CPevalBoxPlot1_Gowalla}}
  \hfill
  \subfloat[LastFM]
    {\includegraphics[scale=0.24]{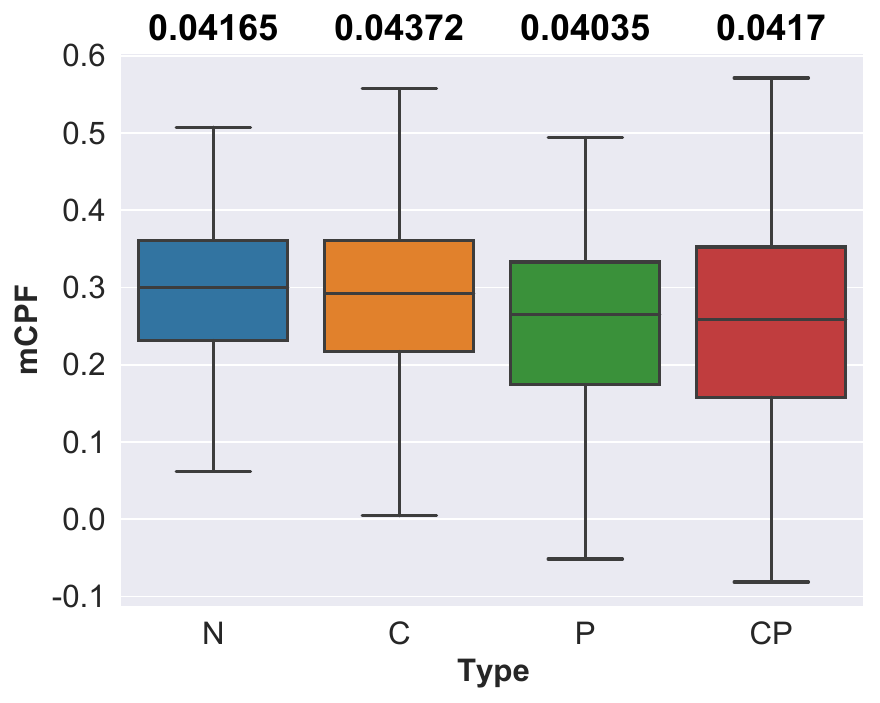}
    \label{fig:CPevalBoxPlot1_LastFM}}
  \hfill
  \subfloat[Foursquare]
    {\includegraphics[scale=0.24]{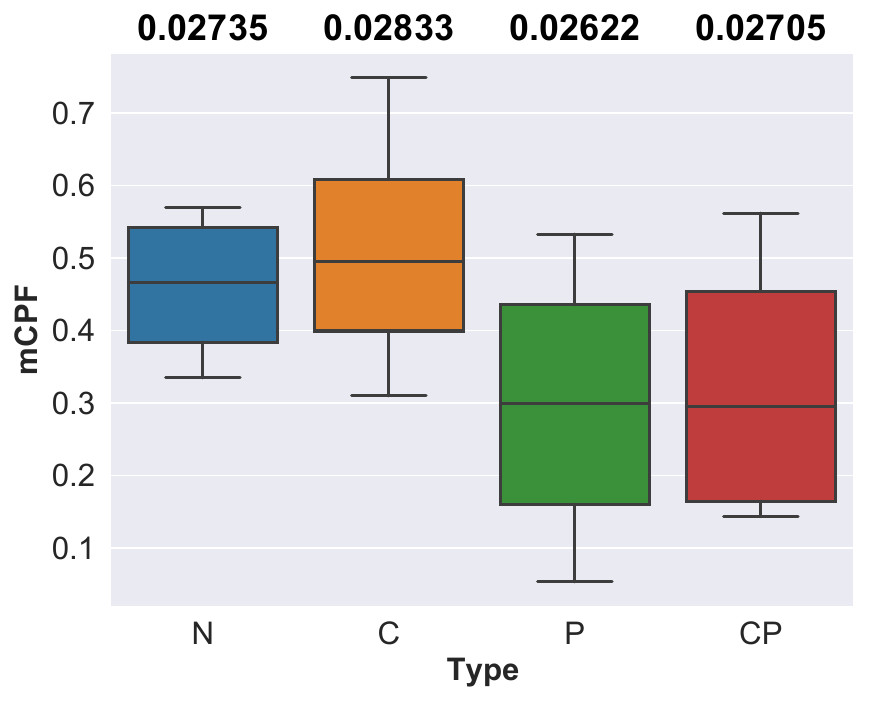}
    \label{fig:CPevalBoxPlot1_Fourquare}}
  \caption{Distributions of mCPF for fairness-unaware (N) and fairness-aware methods (i.e., C, P, and CP) on all 8 datasets. The numbers on top of each plot show the overall performance (i.e., nDCG@10 for all users) according to each fairness methodology.}
\label{fig:CPevalBoxPlotUG1}
\end{figure*}

\begin{figure*}
  \centering
  \subfloat[MovieLens]
    {\includegraphics[scale=0.24]{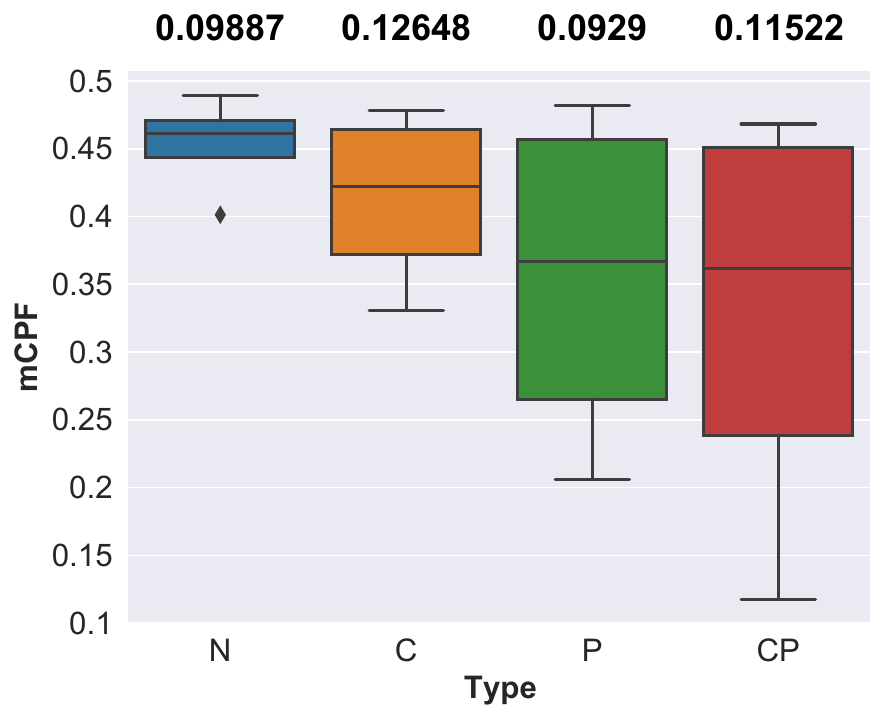}
    \label{fig:CPevalBoxPlot_MovieLens}}
  \hfill
  \subfloat[Epinion]
    {\includegraphics[scale=0.24]{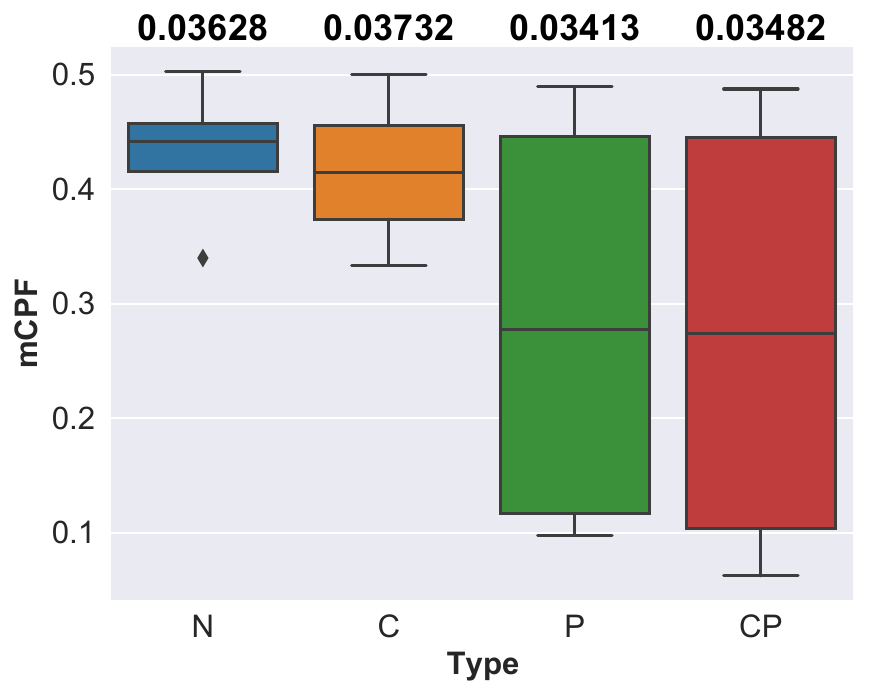}
    \label{fig:CPevalBoxPlot_Epinion}}
  \hfill
  \subfloat[AmazonToy]
    {\includegraphics[scale=0.24]{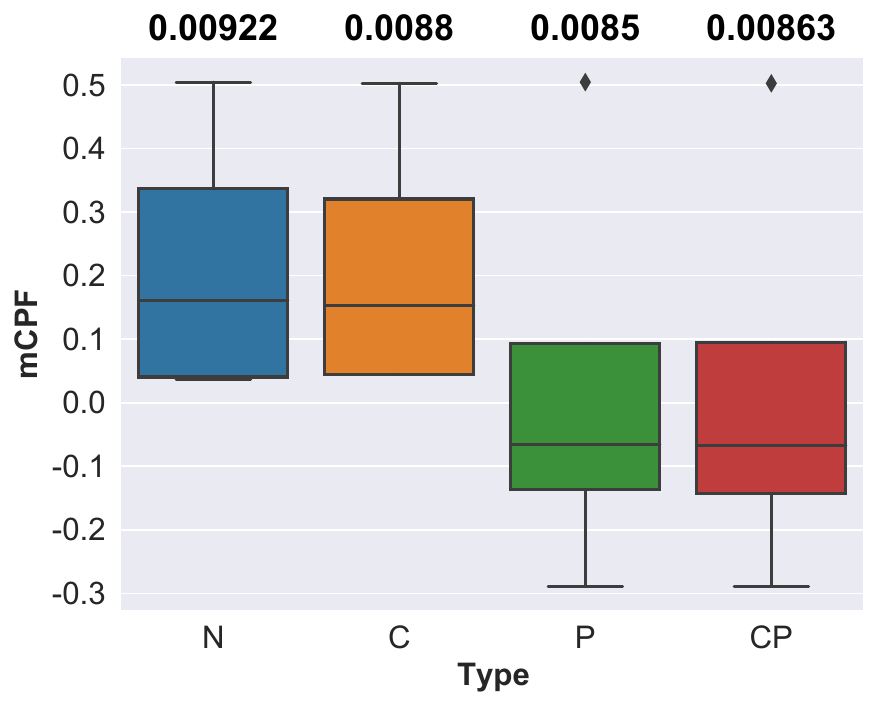}
    \label{fig:CPevalBoxPlot_AmazonToy}}
    \hfill
  \subfloat[AmazonOffice]
    {\includegraphics[scale=0.24]{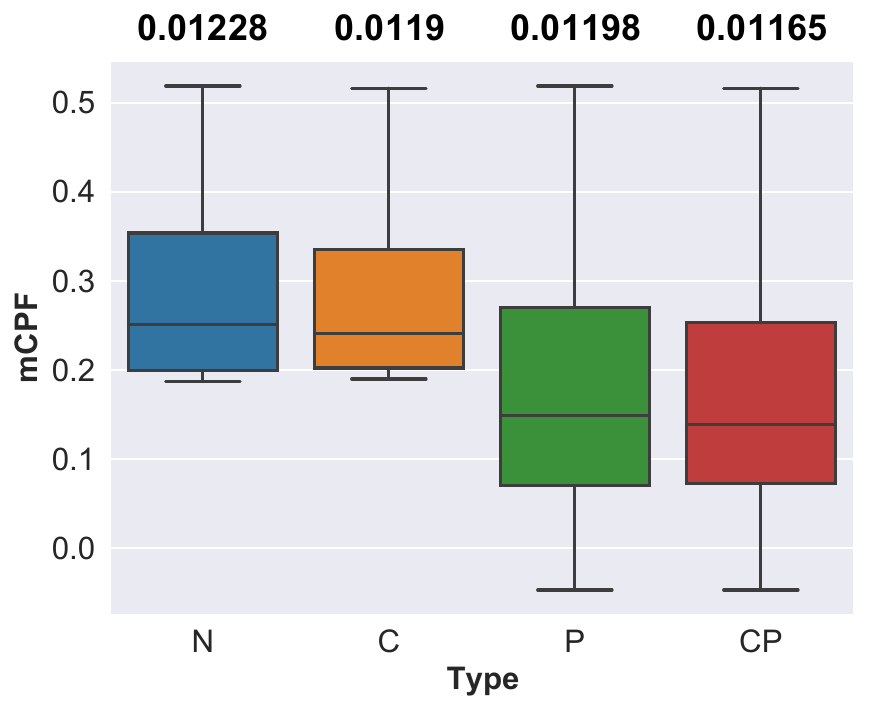}
    \label{fig:CPevalBoxPlot_AmazonOffice}}
   \hfill
  \subfloat[BookCrossing]
    {\includegraphics[scale=0.24]{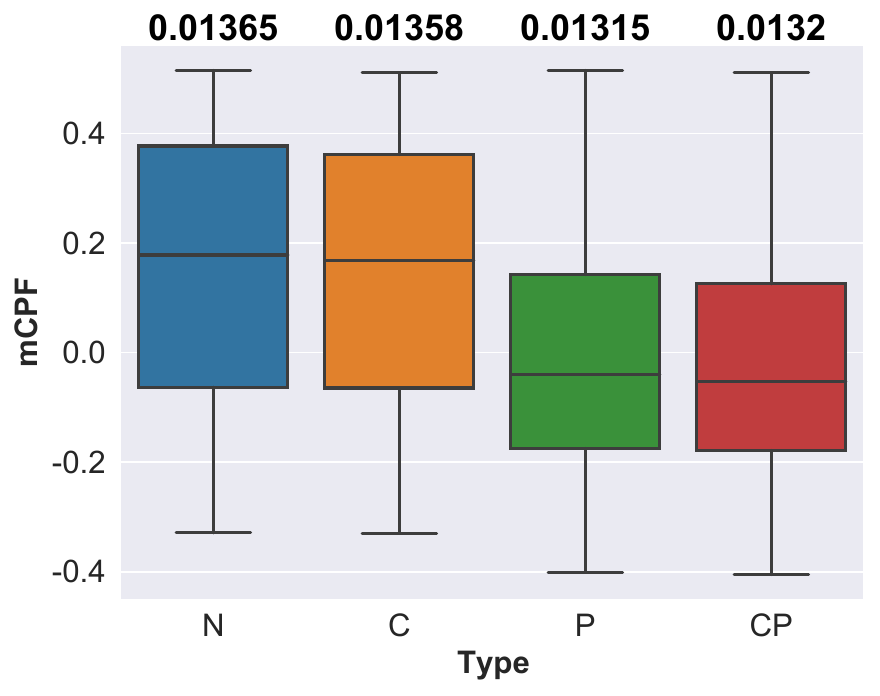}
    \label{fig:CPevalBoxPlot_BookCrossing}}
  \hfill
  \subfloat[Gowalla]
    {\includegraphics[scale=0.24]{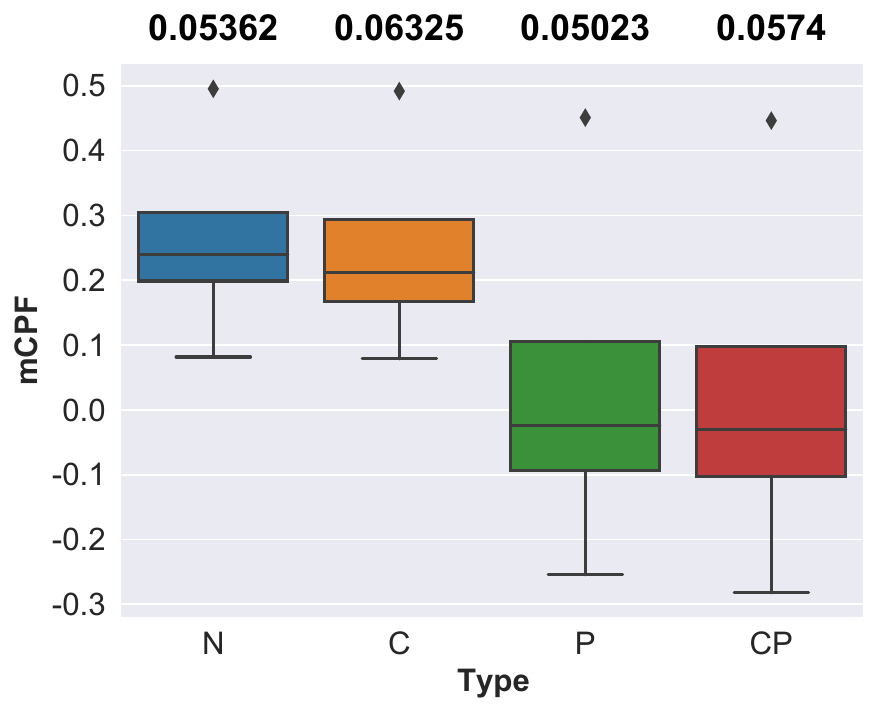}
    \label{fig:CPevalBoxPlot_Gowalla}}
  \hfill
  \subfloat[LastFM]
    {\includegraphics[scale=0.24]{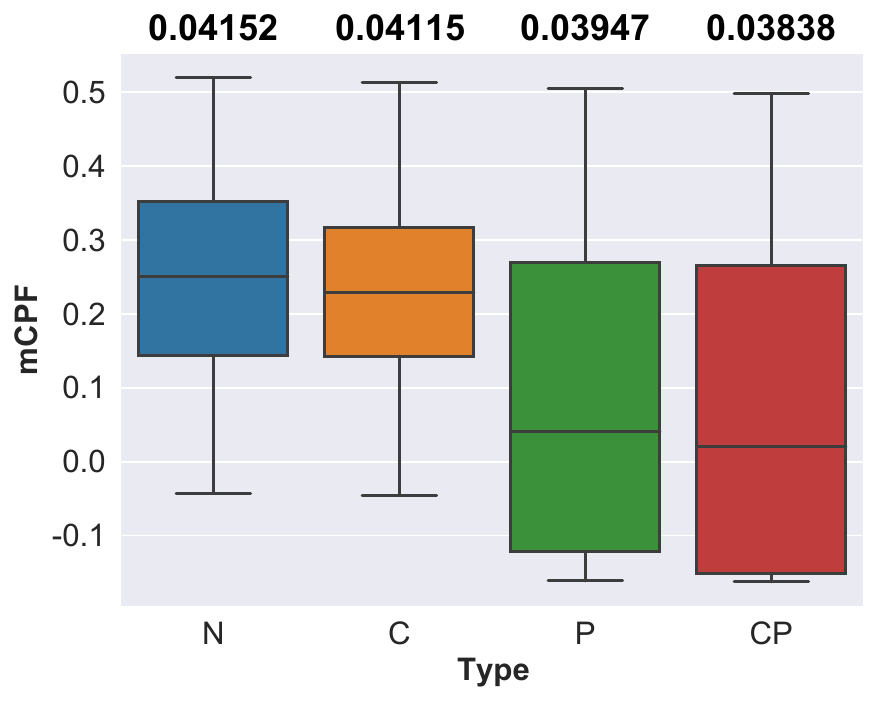}
    \label{fig:CPevalBoxPlot_LastFM}}
  \hfill
  \subfloat[Foursquare]
    {\includegraphics[scale=0.24]{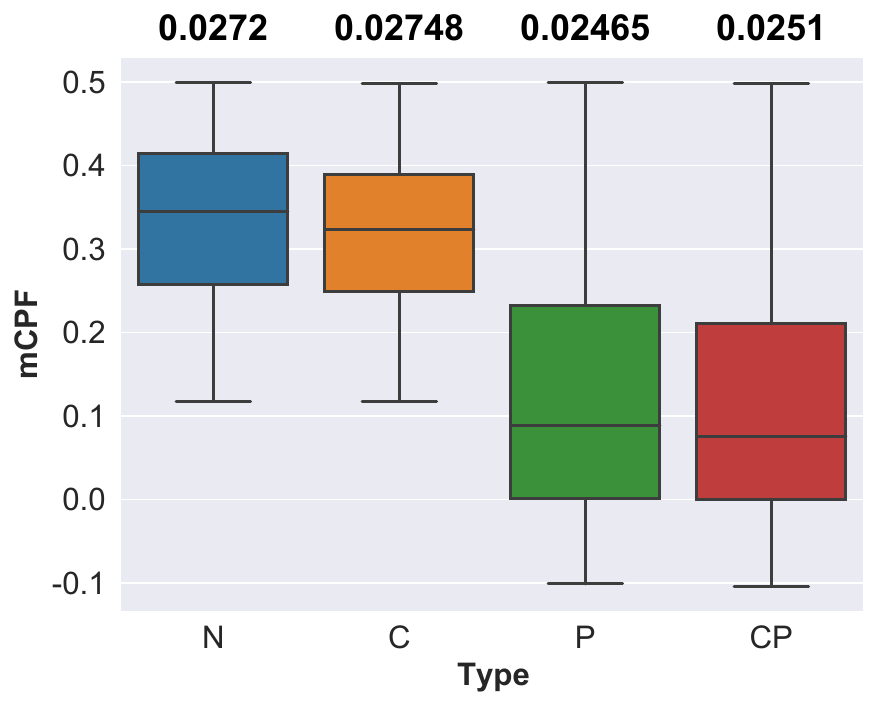}
    \label{fig:CPevalBoxPlot_Fourquare}}
  \caption{Distributions of mCPF for fairness-unaware (N) and fairness-aware methods (i.e., C, P, and CP) on all 8 datasets. The numbers on top of each plot show the overall performance (i.e., nDCG@10 for all users) according to each fairness methodology.}
\label{fig:CPevalBoxPlotUG2}
\end{figure*}

\begin{figure*}
  \centering
  \subfloat[User Groups]
    {\includegraphics[scale=0.24]{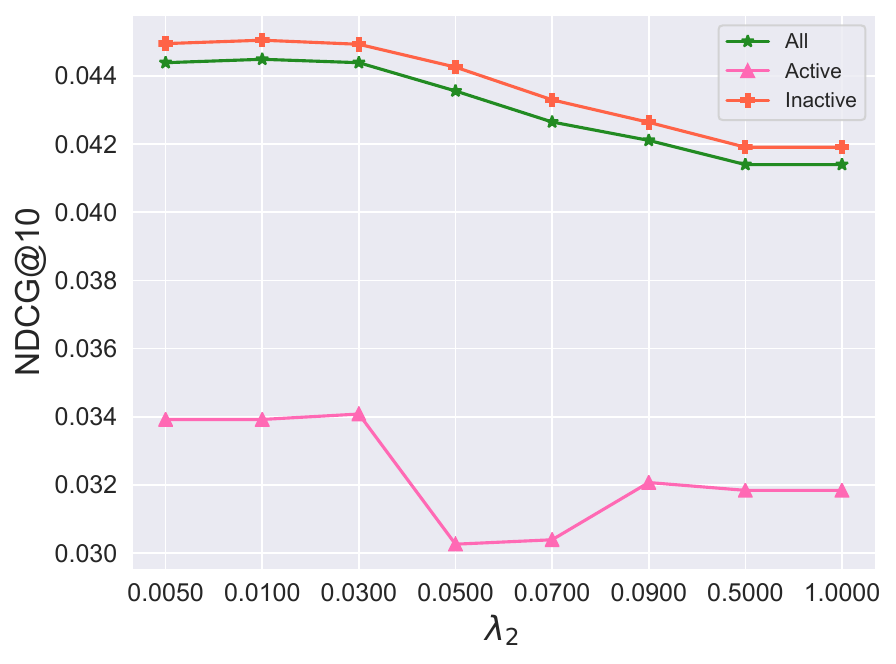}
    \label{fig:ablation_user_ueps}}
  \hfill
  \subfloat[User Groups]
    {\includegraphics[scale=0.24]{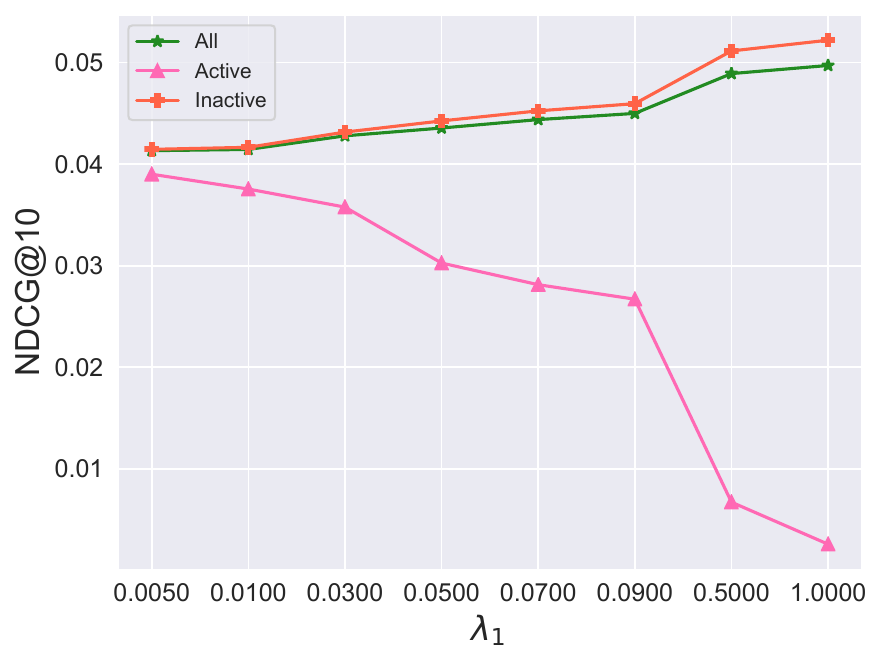}
    \label{fig:ablation_user_ieps}}
    \hfill
  \subfloat[Item Groups]
    {\includegraphics[scale=0.24]{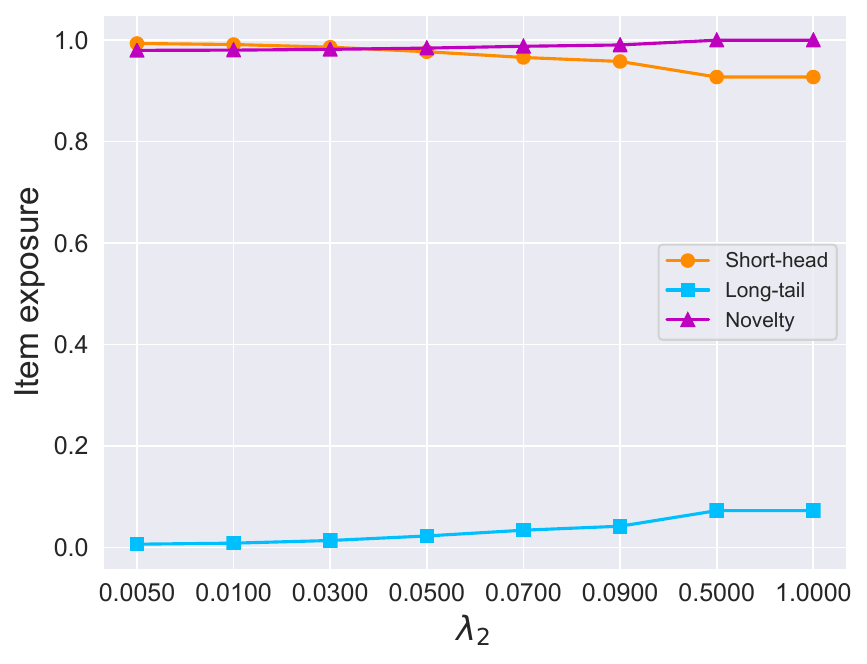}
    \label{fig:ablation_item_ueps}}
   \hfill
  \subfloat[Item Groups]
    {\includegraphics[scale=0.24]{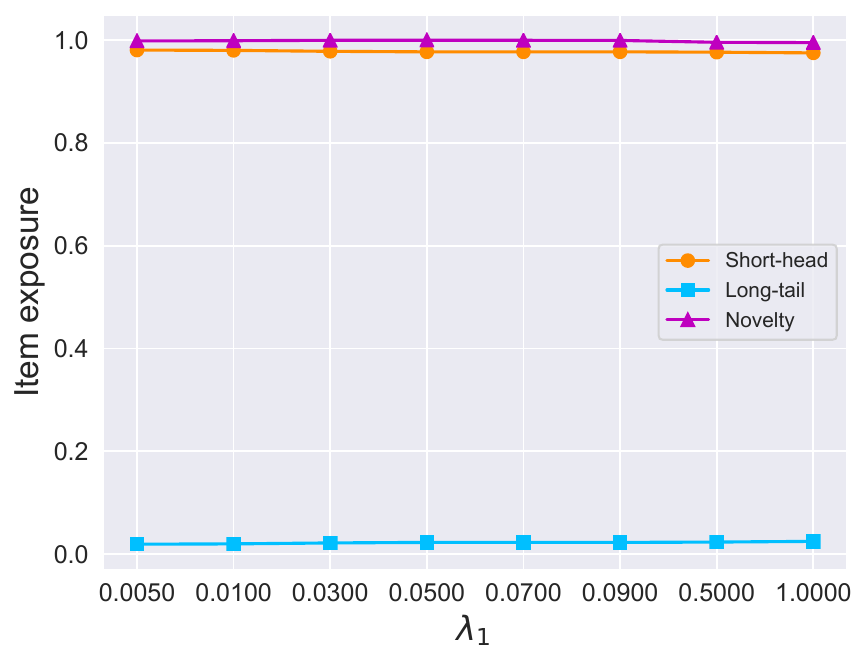}
    \label{fig:ablation_item_ieps}}
  \caption{The metric nDCG@10 and item exposure change on CP-fairness with respect to the $\lambda_1$ and $\lambda_2$ on all, \usergroupA and \usergroupB user groups and \itemgroupA and \itemgroupB item group on activity-level group (UG1). In each figure, the other $\lambda$ is equal to $0.05$.}
\label{fig:ablationLastFMNeuMFUG1}
\end{figure*}

\begin{figure*}
  \centering
  \subfloat[User Groups]
    {\includegraphics[scale=0.24]{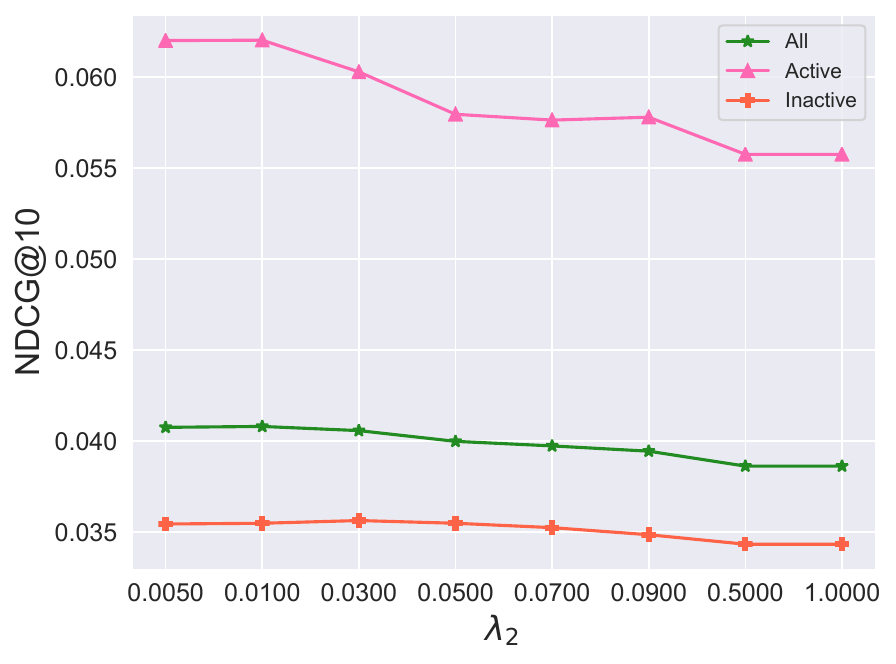}
    \label{fig:ablation_user_ueps}}
  \hfill
  \subfloat[User Groups]
    {\includegraphics[scale=0.24]{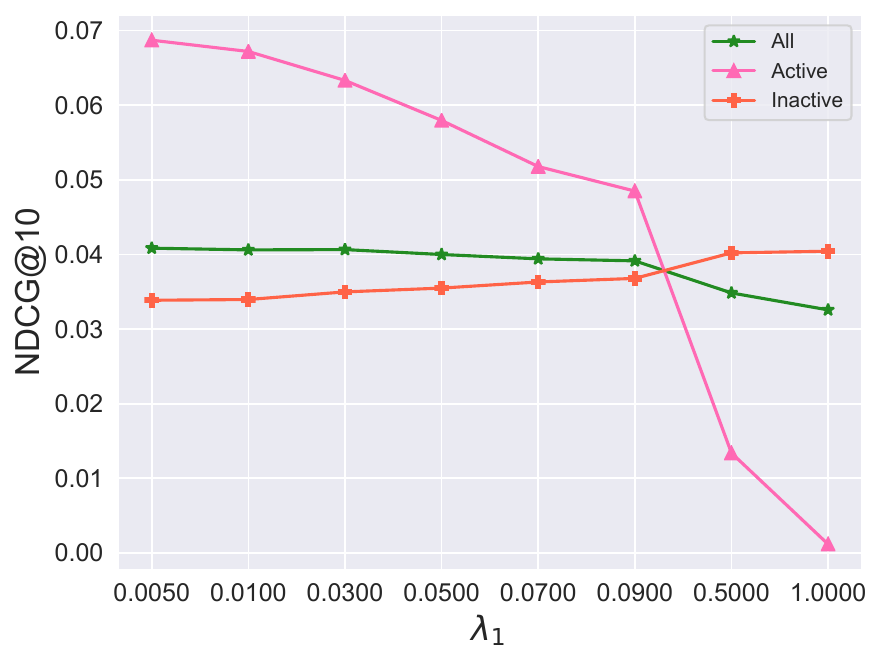}
    \label{fig:ablation_user_ieps}}
    \hfill
  \subfloat[Item Groups]
    {\includegraphics[scale=0.24]{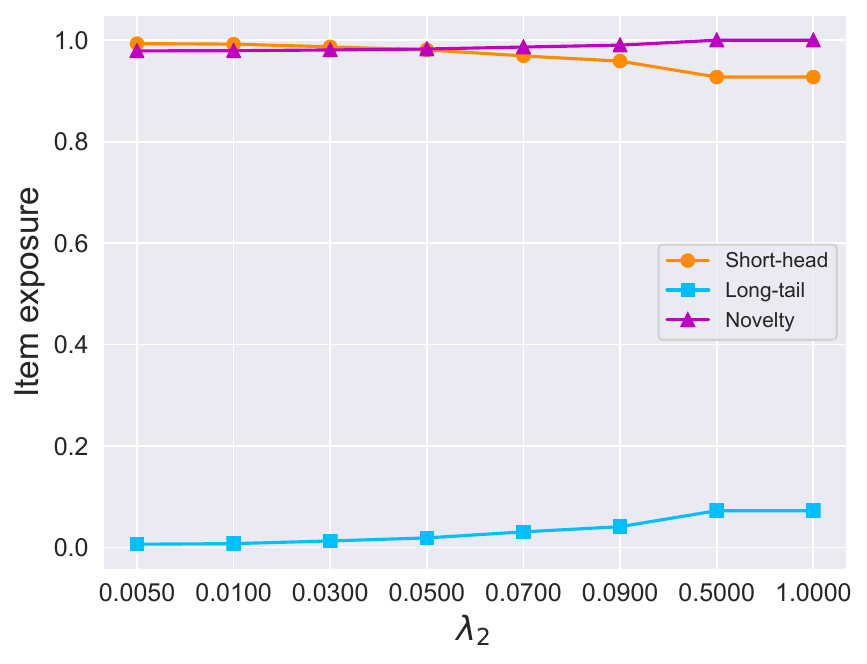}
    \label{fig:ablation_item_ueps}}
   \hfill
  \subfloat[Item Groups]
    {\includegraphics[scale=0.24]{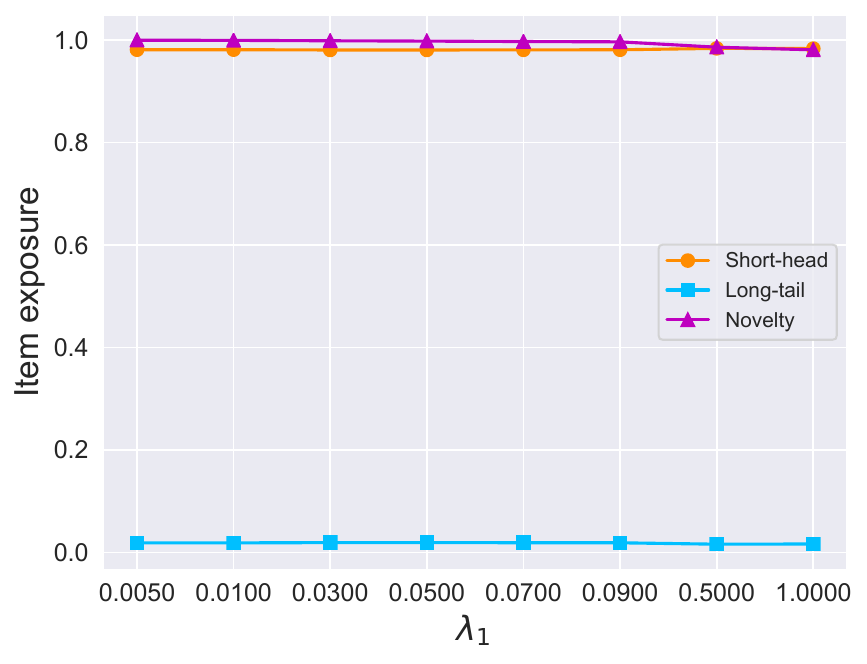}
    \label{fig:ablation_item_ieps}}
  \caption{The metric nDCG@10 and item exposure change on CP-fairness with respect to the $\lambda_1$ and $\lambda_2$ on all, \usergroupA and \usergroupB user groups and \itemgroupA and \itemgroupB item group on mainstream group (UG2). In each figure, the other $\lambda$ is equal to $0.05$.}
\label{fig:ablationLastFMNeuMFUG2}
\end{figure*}

The results presented in Figures \ref{fig:ablationLastFMNeuMFUG1} and \ref{fig:ablationLastFMNeuMFUG2} indicate that $\lambda_1$ exhibits a more \dquotes{accuracy-centric} behavior, implying that this parameter affects user-fairness and overall system accuracy with little impact on the items' expositions. The similarity in the type of items used by both active and inactive users suggests that there is no significant difference in their preferences for popular items. However, the $\lambda_2$ parameter showcases an ``exposure-centric'' behavior that can affect both the accuracy and beyond-accuracy of the system. An excessive increase in $\lambda_2$ may not necessarily improve overall accuracy as many long-tail items lack sufficient interaction data, leading to uncertainty in their potential impact on user satisfaction if included in recommendation lists. Thus, finding the optimal balance when selecting model parameters is crucial for enhancing the fairness and overall effectiveness of the marketplace \cite{wu2022multi}.

\subsection{Additional Insight: The Implications of the \texttt{CP-FairRank} on Items and Users within Subgroups}

\begin{figure}
  \centering
  \subfloat[Distribution of Utility of the top-$k$ rec. lists measured using nDCG.]
    {
    \includegraphics[scale=0.40]{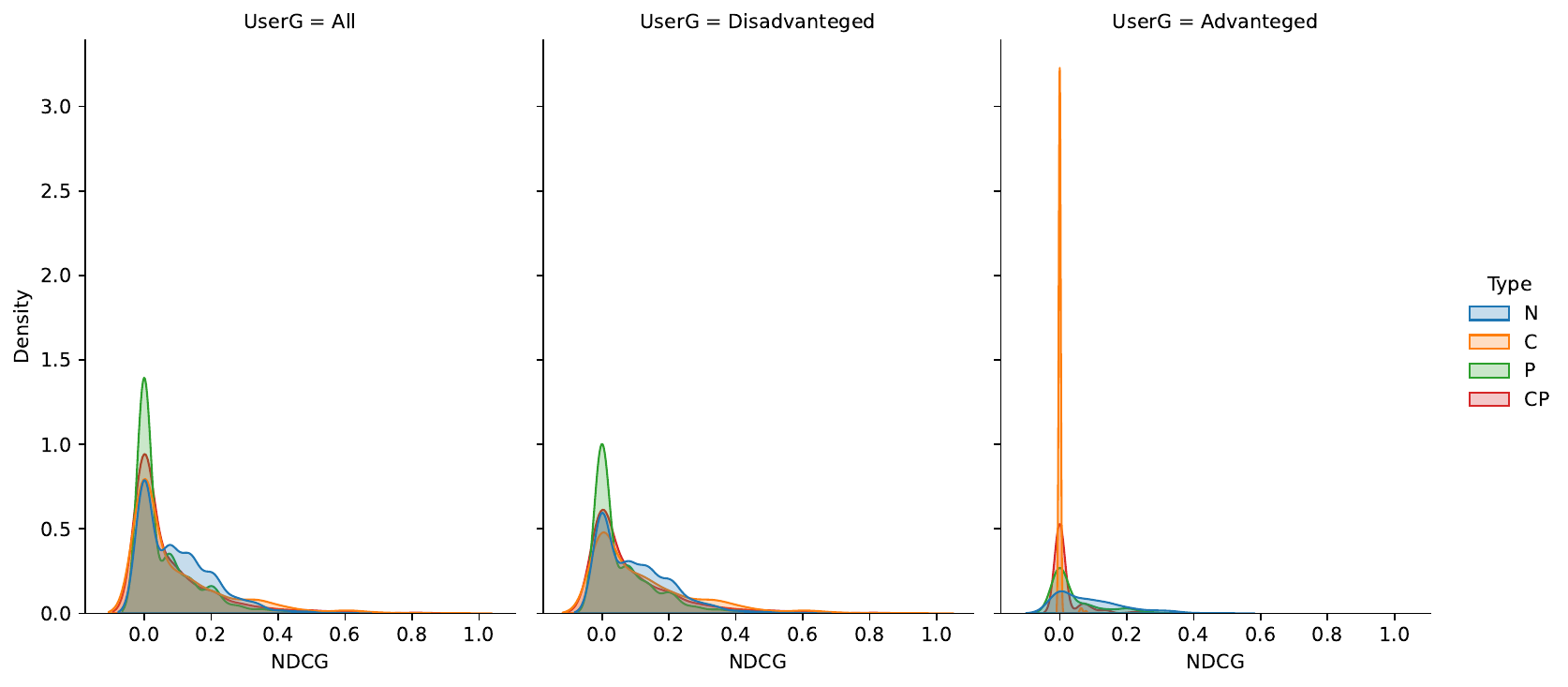}
    \label{fig:distribution_ndcg}
    }
  \hfill
  \subfloat[Distribution of Novelty of the top-$k$ rec.~lists]
    {
    \includegraphics[scale=0.40]{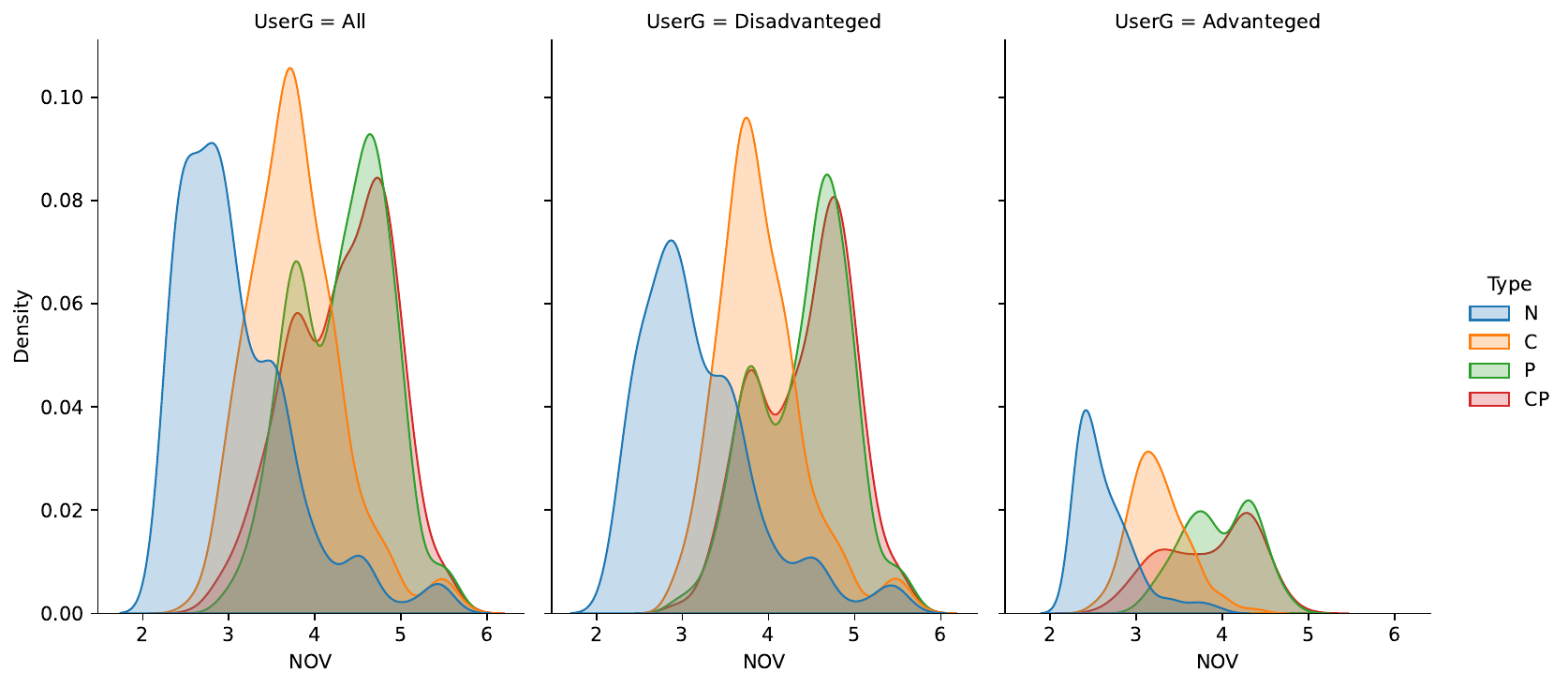}
    \label{fig:distribution_nov}
    }
    \caption{Distribution of the utility (a) and novelty (b) of the top-$k$ recommendation list across users}
      \label{fig:distribution}
\end{figure}

\updated{In this section, we study further the implications of our proposed CP method. We have previously demonstrated its prowess in enhancing both consumer and producer fairness without compromising, and sometimes even augmenting, the overall utility of recommendations. Drawing inspiration from \cite{wu2021tfrom}, where authors highlighted potential trade-offs between fairness evaluation objectives and overall recommendation utility, we aim to understand the potential impact of our algorithm on the exposure of items within various subgroups. Specifically, we aim to offer a meticulous and in-depth analysis regarding the impact on both advantaged and disadvantaged groups.}

\updated{Toward this goal, while Wu et al.~\cite{wu2021tfrom} employed metrics to track variations in item exposure and recommendation quality, we shift our focus to distribution representations. As illustrated in Figures \ref{fig:distribution_ndcg} and \ref{fig:distribution_nov} we present the distribution of recommendation accuracy (NDCG) across users, as well as the novelty distribution of recommended items for the entire users. It is noteworthy that we use novelty as a proxy for item exposure, complementing it as an additional metric beyond accuracy, in line with the exposure metric (recommendation count) employed in prior sections to measure provider fairness.}

\updated{Upon examining the second figure, it becomes evident that both methods, P and CP, effectively push the novelty of all items to the right, indicating an increase in item novelty. This observation holds true for both advantaged and disadvantaged groups, as well as overall. Interestingly, even when comparing C with N, method C improves the novelty of recommended items, despite not being explicitly designed for this purpose. These findings suggest that under our proposed method, producers benefit from increased exposure to their less popular items, and fairness is achieved at the group level, as demonstrated in earlier sections. In terms of the impact on the novelty of items, the hierarchy can be summarized as \textbf{$CP=P>C>N$}. Shifting our attention to the primary query: does this enhancement in fairness and novelty come at the cost of recommendation quality, especially when considering specific user subgroups? An analysis of Fig.~\ref{fig:distribution} reveals that methods C, P, and CP all experience a decline in recommendation quality for both groups. However, the decline is most pronounced for the C method. Hence, when focusing on recommendation quality, the methods can be ranked as \textbf{$N>CP=P>C$.}}

\updated{Overall, our study confirms the previous findings that there exists an inherent trade-off between recommendation quality and the attainment of fairness in recommendation systems. However, our proposed method adeptly addresses the shortcomings of the consumer fairness method proposed by Li et al. \cite{li2021user}, effectively minimizing the adverse impacts associated with this trade-off.}
\section{Conclusions}
\label{sec:conclusion}
In this paper, we study fairness-aware recommendation algorithms from both the user and item perspectives. We first show that current recommendation algorithms produce unfair recommendations between different disadvantaged user and item groups due to the natural biases and imperfections in the underlying user interaction data. \newtxt{To address these issues, we propose a personalized CP-fairness-constrained re-ranking method to compensate for the biases against long-tail items and cold-start users, which are not necessarily groups defined by sensitive attributes. Our method seeks to balance the representation in the recommendation space, while maintaining the recommendation quality.} 
Extensive experiments across eight datasets on two different user grouping methods, namely, activity-level and mainstreamness, indicate that our method can reduce the unfair outcomes on beneficiary stakeholders, consumers and providers while improving the overall recommendation quality. This demonstrates that algorithmic fairness may play a critical role in mitigating data biases that, if left uncontrolled, can result in societal discrimination and polarization of opinions and retail businesses. \newtxt{Although this study has focused on specific user and item groupings, we postulate that our findings could extend to other groupings (e.g., based on sensitive attributes) that introduce substantial disparities, as our re-ranking method's objective function is designed to optimize a balance between fairness and accuracy, though this could come at the expense of accuracy depending on the particularities of the grouping methodology and the feasibility of achieving such a balance.} Hence, in the future, we intend to extend this work by analyzing other fine-grained fairness circumstances, such as \textit{mis-calibration} and \textit{individual} fairness, to determine if our system is capable of addressing fairness constraints on these new definitions of fairness and grouping methodology. \updated{Moreover, we anticipate refining our model to better address the complexities of multi-sided fairness, in recognition of the significant challenges this presents, as highlighted in the discussion on uni-sided and multi-sided fairness optimization. Furthermore, we recognize the necessity of testing our algorithms on larger, more diverse datasets in future studies, as this will provide further insights into their scalability and effectiveness in real-world, large-scale scenarios \cite{rahmani2022exploring,deldjoo2018content}, e.g., in temporal settings \cite{rahmani2022exploring}.} In addition, we believe it would be interesting to analyze the biases induced in each domain based on dataset characteristics (such as popularity bias, the user gini coefficient representing user activity distribution, and sparsity) and to explain the results based on these various characteristics (see, for example, our recent research in this line~\cite{deldjoo2021explaining}). Another research domain that we consider highly pertinent to this discussion is the study of fairness in the context of generative AI and large language models (LLMs) applied to recommender systems and machine learning at large \cite{deldjoo2024understanding,deldjoo2023fairnesschatgpt,DBLP:conf/recsys/ZhangBZWF023, Nazary2023ChatGPT,weidinger2021ethical}. Given that these systems are trained on vast amounts of internet data, it becomes crucial to examine their potential biases and to consider strategies for mitigating social and ethical risks, as well as optimizing overall performance.

\subsection{Limitations}
\newtxt{Despite proposing a practical system to seamlessly and jointly address CO-Fairness in recommendations, like any research, our work is not without limitations. The primary constraints of this study, particularly the selection of \textit{fairness baselines,} are discussed in detail in the subsequent sections.}

\newtxt{Finding suitable baselines for our work presented several challenges, including 
 (i) the absence of methods that holistically address CP-Fairness (most research has been focused on single-stakeholder constraints), and (ii) varying interpretations and assumptions on \dquotes{what is fair}. For the first point, most existing work either focuses on evaluating fairness on CP, often without coherent algorithmic integration \cite{deldjoo2023fairness,amigo2023unifying}. }

\newtxt{The second concern,  defining ``what is fairness'' in a group setting in recommendation algorithms, as highlighted in \cite{deldjoo2023fairness} and further underscored by \citet{kirnap2021estimation}, is multifaceted and diverse. In the RecSys literature, we can broadly consider the following notions:} 

\begin{enumerate}
    \item \newtxt{\textbf{Fair target representation}: This is where fairness hinges on representing targets (such as user groups, item genres, etc.) in a manner that is either equal (everyone gets the same representation) or proportional (representation is aligned with some inherent proportion in the dataset or the real world). The majority of algorithms focused on fairness enhancement, irrespective of the stakeholder, belong to this category.}
    
    \item \newtxt{\textbf{Guaranteed allocation}: This perspective ensures that every identifiable group receives a spot in the recommendation list, thus ensuring that no group is entirely left out or overlooked. It is rooted in the principle that fairness is not only about balanced representation but also about preventing complete absence, ensuring that every group, regardless of its size or nature, is acknowledged. For instance, the work \citet{wu2021tfrom} working on CP-fairness operates on this principle.}
    
\end{enumerate}

\newtxt{Our CP framework primarily focuses on the first approach, that is, achieving fairness through fair target representation. Integrating other CP-Fairness perspectives such \cite{wu2021tfrom}, which belongs to the guaranteed allocation category, into our framework is not straightforward. Such an adaptation would necessitate a comprehensive revision of how we define fairness in our method. For these reasons, we decided to compare our proposed fair CP re-ranking algorithm with unilateral fairness models, C and P-fairness re-ranking on top of the core base ranker model to show how our method can achieve the desirable performance on fairness metrics in both stakeholders and overall recommendation performance (accuracy and beyond-accuracy). Specifically, the C variation of our system targets consumer-side fairness objectives, aligned with the approach introduced by~\citet{li2021user}. Our approach incorporates users' parity objectives in an unconstrained optimization problem and employs a greedy algorithm for optimal polynomial-time solutions. On the other hand, the P variation emphasizes provider-side fairness objectives.}


\bibliographystyle{ACM-Reference-Format}
\bibliography{sample-base}


\end{document}